%% file: main.tex
\RequirePackage{fix-cm}

\documentclass{ieeeaccess}
\usepackage{cite}
\usepackage{amsmath,amssymb,amsfonts,amsthm}
\usepackage{algorithm}
\usepackage{algorithmic}
\usepackage{graphicx}
\usepackage{textcomp}
\usepackage{booktabs}
\usepackage{multirow}
\usepackage{url}
\usepackage{hyperref}
\usepackage{bm}

\usepackage{silence}

\makeatletter
\AtBeginDocument{\DeclareMathVersion{bold}
\SetSymbolFont{operators}{bold}{T1}{times}{b}{n}
\SetSymbolFont{NewLetters}{bold}{T1}{times}{b}{it}
\SetMathAlphabet{\mathrm}{bold}{T1}{times}{b}{n}
\SetMathAlphabet{\mathit}{bold}{T1}{times}{b}{it}
\SetMathAlphabet{\mathbf}{bold}{T1}{times}{b}{n}
\SetMathAlphabet{\mathtt}{bold}{OT1}{pcr}{b}{n}
\SetSymbolFont{symbols}{bold}{OMS}{cmsy}{b}{n}
\renewcommand\boldmath{\@nomath\boldmath\mathversion{bold}}}
\makeatother

\newtheorem{proposition}{Proposition}
\newtheorem{lemma}{Lemma}

\def\BibTeX{{\rm B\kern-.05em{\sc i\kern-.025em b}\kern-.08em
    T\kern-.1667em\lower.7ex\hbox{E}\kern-.125emX}}

\begin{document}
\history{Accepted for publication in IEEE Access.}
\doi{10.1109/ACCESS.2026.3687163}

\title{DC-Ada: Reward-Only Decentralized Sensor Adaptation for Heterogeneous Multi-Robot Teams}
%DC-Ada: Reward-Only Decentralized Observation-Interface Adaptation for Heterogeneous Multi-Robot Teams}

%\author{\uppercase{Saad Alqithami}\authorrefmark{1}}
%\address[1]{Affiliation to be updated (e-mail: to be updated)}
\author{\uppercase{Saad Alqithami} \IEEEmembership{Member, IEEE}}
\address{Department of Computer Science, Al-Baha University, Albaha 65779 Saudi Arabia (e-mail: salqithami@bu.edu.sa)}

%\tfootnote{Code and experimental pipeline: \url{https://github.com/alqithami/DC-ADA}. The repository contains the full simulator, training scripts, and reproducibility utilities used to generate the results in this paper.}

\markboth
{Alqithami: DC-Ada for Heterogeneous Multi-Robot Teams}
{Alqithami: DC-Ada for Heterogeneous Multi-Robot Teams}

%\corresp{Corresponding author: Saad Alqithami (e-mail: salqithami@bu.edu.sa).}

\begin{abstract}
Heterogeneity is a defining feature of deployed multi-robot teams: platforms often differ in sensing modalities, ranges, fields of view, and degradation or failure patterns. A persistent deployment challenge is that controllers trained under a nominal sensing configuration can degrade sharply when executed on robots with missing or mismatched sensing, even when the task and action interface remains unchanged. This paper presents \textbf{DC-Ada}, a reward-only decentralized adaptation method that keeps a pretrained \emph{shared policy} frozen and instead adapts compact, per-robot \emph{observation transforms} to reconcile heterogeneous sensing with a fixed inference interface. DC-Ada is gradient-free and communication-minimal: adaptation proceeds via budgeted accept/reject random search, using short common-random-number rollouts to compare candidate perturbations under a strict environment-step budget. We evaluate DC-Ada and four baselines---shared policy, observation normalization, random perturbation, and local fine-tuning---in a lightweight deterministic two-dimensional multi-robot simulator spanning warehouse logistics, search-and-rescue, and collaborative mapping. The main evaluation covers four heterogeneity regimes (H0--H3) and five random seeds under a matched budget of $200{,}000$ joint environment steps per run. We report shaped return, thresholded task completion, continuous progress metrics, threshold-sensitivity analyses, and runtime/scalar-feedback accounting. We further include targeted severe-heterogeneity analyses: H3 design ablations and observation-stress tests in mapping and search-and-rescue, together with a ten-seed clean-H3 evaluation across all three environments. Across domains, heterogeneity substantially affects the frozen shared policy, but no single mitigation strategy dominates every task and metric. Observation normalization is strongest for reward robustness in warehouse logistics and remains competitive in search-and-rescue, while the frozen shared policy is strongest for shaped reward in collaborative mapping, with local fine-tuning close behind. DC-Ada provides a complementary operating point: it improves completion performance most clearly in coverage-based mapping under severe heterogeneity, while requiring only scalar team returns and avoiding policy fine-tuning or persistent message exchange. The H3 ablations indicate that common-random-number evaluation and multi-candidate selection are important for DC-Ada's mapping-completion gains. The observation-stress tests show that these gains persist under added noise, dropout, delay, and mild combined perturbations. The ten-seed clean-H3 evaluation preserves the same domain-specific pattern: mapping remains DC-Ada's clearest completion-oriented advantage, warehouse remains reward-dominated by observation normalization, and search-and-rescue remains method-dependent. These results position DC-Ada as a practical deploy-time observation-interface adaptation mechanism for heterogeneous teams when gradients, privileged state, or high-bandwidth communication are unavailable or undesirable.
\end{abstract}

\begin{keywords}
Multi-robot systems, heterogeneous sensing, observation-interface adaptation, test-time adaptation, reward-only learning, derivative-free optimization, communication constraints, multi-agent reinforcement learning.
\end{keywords}

\titlepgskip=-15pt

\maketitle
% ============================================================

\section{Introduction}
\label{sec:introduction}
Multi-robot teams are increasingly deployed in domains where parallelism, spatial coverage, and redundancy translate directly into mission value, including disaster response and search operations \cite{murphy2019disaster}, autonomous water-quality monitoring \cite{carrillo2022lake}, heterogeneous warehouse logistics \cite{Kang02122025}, and planetary rover exploration \cite{schenker2003planetary}. In these deployments, \emph{heterogeneity} is the norm rather than the exception: robots differ in sensing modalities, sensing range, fields of view, compute budgets, and failure modes. Foundational taxonomies and formal models emphasize that heterogeneous capability fundamentally shapes coordination, task allocation, and robustness \cite{dudek1993taxonomy,gerkey2004formal}. In learning-enabled teams, heterogeneity frequently manifests as asymmetric sensing and partial degradations, producing persistent information asymmetry and deployment-time distribution shift that can invalidate assumptions made during training \cite{bettini2023heterogeneous,gao2023asymmetric,fu2022robust,lin2020multi}.

\subsection{Motivation: the deployment gap under heterogeneous sensing}
A common practice in learning-based multi-robot control is to train a \emph{shared policy} under a nominal sensing configuration (often homogeneous) and then deploy it across a team. This practice leverages shared-policy generalization and simplifies deployment, but it can introduce systematic mismatch when deployed robots exhibit modality loss, reduced range, degraded resolution, or altered sensing geometry. Importantly, the resulting failure is often not a minor reduction in measurement quality but a \emph{semantic} mismatch in the policy input: for example, segments corresponding to missing range measurements may be interpreted as free space, or a controller trained with long-range sensing may behave unsafely when deployed with shortened effective sensing range. These mismatches are particularly consequential in partially observable tasks where the policy implicitly relies on specific cue reliabilities and observation semantics.

Multi-agent reinforcement learning (MARL) provides powerful coordination frameworks across a spectrum of assumptions, including independent learners \cite{matignon2012independent}, centralized training with decentralized execution \cite{lowe2017multi}, and cooperative value factorization \cite{rashid2018qmix}. While these approaches can learn sophisticated behaviors, they commonly assume consistent observation interfaces between training and deployment, and they often benefit from centralized training machinery and/or structured information exchange \cite{gupta2017cooperative,nguyen2020deep}. In addition, many fielded multi-robot systems operate under practical communication constraints: low-bandwidth links and intermittent connectivity make high-rate exchange of raw observations, dense latent messages, or gradients undesirable. This motivates classical and modern work on communication-efficient coordination, spanning consensus and distributed control \cite{choi2009consensus,atanasov2015decentralized}, distributed optimization under bandwidth constraints \cite{marcotte2020optimizing,patwardhan2024distributed}, and event-triggered communication policies \cite{hu2021event}. Learned communication remains an active research direction, beginning with differentiable communication in MARL \cite{foerster2016learning} and extending to attention/message-passing architectures \cite{phan2023attention,zhang2023commnetx}, structured coordination layers \cite{Meng2023DC2Net,Yang2024TeamComm}, and high-level language-augmented planning and coordination \cite{Rajvanshi2025SayCoNav,Rana2023SayPlan,Zhao2024HAS}. Despite this progress, many deployments prefer approaches that do not require sustained high-rate messaging and do not introduce new inter-robot protocols that must be validated.

\subsection{Data-centric adaptation: adapt the interface, not the policy}
This paper targets a practical deployment regime in which (i) the task and action interface remain unchanged, (ii) a high-performing shared policy is already available, and (iii) robots differ in sensing. Rather than adapting the policy parameters---which can be expensive to validate and may raise safety and stability concerns---we pursue a data-centric alternative: \emph{keep the shared policy frozen} and adapt only what each robot feeds into it. Concretely, each robot learns a small local transformation that maps its fixed-layout observation vector into an inference interface that the frozen shared policy can interpret more robustly under heterogeneous sensing.

The approach aligns with modular sim-to-real transfer and test-time adaptation paradigms in robotics, where compact front-end modules are adapted to reconcile distribution shift while downstream controllers remain stable \cite{finn2017model,rusu2017sim2real,kalashnikov2018scalable}. It is also consistent with data-centric AI perspectives that emphasize improving performance by shaping data, representations, and adaptation interfaces rather than repeatedly retraining full policies \cite{zha2025data}. Finally, it admits a reward-only adaptation loop: when gradients or privileged states are unavailable on-device, derivative-free optimization based on scalar rollouts remains applicable \cite{spall1992multivariate,kushner2003stochastic,borkar2008stochastic}. DC-Ada adopts a conservative accept/reject mechanism inspired by random search and evolution strategies \cite{salimans2017es,mania2018simple}, and is interpreted through the lens of bandit/zeroth-order optimization \cite{flaxman2005bandit,nesterov2017random,duchi2015optimal}. At a conceptual level, the interface adapter can be viewed as an engineered analogue of reliability-weighted sensing: when cue reliability changes, biological systems reweight and reshape sensing rather than relearn the entire controller \cite{ernst2002optimal,peterka2002sensorimotor,vanbeers2002feeling}.

Figure~\ref{fig:dcada_overview} summarizes the proposed deployment pattern. Each robot applies a compact local adapter to its fixed-layout observation vector, while the shared policy remains frozen. The only feedback required to drive adaptation is a scalar team-level rollout return, which can be communicated at very low rate (once per rollout), avoiding sustained message passing during execution.

\begin{figure}[t]
\centering
\includegraphics[width=\linewidth]{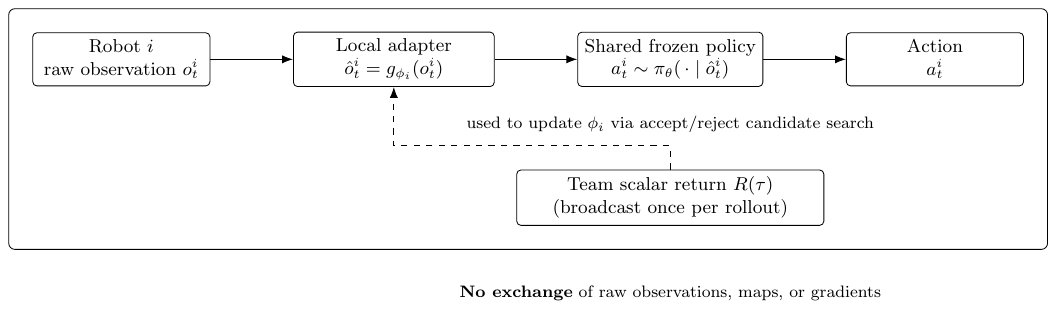}
\caption{DC-Ada deployment and adaptation interface. Each robot adapts only a local observation transform $g_{\phi_i}$ in front of a shared frozen policy $\pi_\theta$. Adaptation uses scalar rollout return broadcasts $R(\tau)$ (once per rollout) and does not require sharing raw observations, maps, or gradients.}
\label{fig:dcada_overview}
\end{figure}

\subsection{Contributions}
We make the following contributions.
First, we formulate heterogeneous-sensing adaptation in cooperative multi-robot teams as learning robot-specific observation transforms in front of a frozen shared policy, and we enforce a fixed-layout observation interface to prevent checkpoint incompatibility and confounded comparisons across heterogeneity levels. Second, we propose \textbf{DC-Ada}, a decentralized reward-only adaptation method that performs conservative accept/reject zeroth-order updates using short common-random-number (CRN) rollouts under a strict environment-step budget. Third, we provide an analysis and design rationale connecting DC-Ada to zeroth-order and bandit optimization, motivating CRN as variance reduction and clarifying the reliability implications of accept/reject selection under bounded returns. Finally, we provide a transparent controlled-simulation evaluation of DC-Ada and four baselines in a lightweight deterministic two-dimensional multi-robot simulator spanning warehouse logistics, search-and-rescue, and collaborative mapping across four heterogeneity regimes (H0--H3) and five seeds. We report reward, success, continuous progress metrics, threshold sensitivity, targeted H3 design ablations, observation-stress tests, an extended ten-seed clean-H3 confirmation across all three environments, and measured runtime/scalar-feedback accounting under a matched $200{,}000$-step budget.

% ============================================================
\section{Related Work}
\label{sec:related_work}
DC-Ada lies at the intersection of heterogeneous multi-robot systems, communication-limited coordination, deploy-time adaptation, and derivative-free optimization. Prior work spans (i) capability-aware modeling and coordination in heterogeneous teams, (ii) consensus, distributed estimation, and communication-efficient mapping, (iii) multi-agent reinforcement learning with and without learned communication, (iv) sim-to-real and data-centric robotics aimed at robustness under distribution shift, and (v) bandit/zeroth-order optimization methods that enable reward-only improvement without gradients. The closest conceptual neighbors to DC-Ada are approaches that seek robustness to sensor mismatch and partial failures while preserving decentralized execution, but DC-Ada differs in that it adapts \emph{only} a compact observation interface using \emph{scalar} return feedback, avoiding both centralized state and persistent inter-agent messaging.

\subsection{Heterogeneity and asymmetric capability in multi-robot teams}
Heterogeneity is intrinsic to multi-robot systems, spanning sensing, actuation, compute, power constraints, and communication capabilities. Early foundational work formalized this diversity by proposing taxonomies of multi-robot teams and studying how capability differences affect task allocation and coordination \cite{dudek1993taxonomy}. Subsequent formal models of multi-robot task allocation and coalition formation emphasized that heterogeneous execution requires explicit reasoning about robot capabilities and constraints \cite{gerkey2004formal}. In modern deployments, heterogeneity frequently manifests as \emph{asymmetric sensing} and partial degradations, where robots receive non-equivalent information about the same underlying world and therefore face different partial observability regimes \cite{bettini2023heterogeneous,gao2023asymmetric}. Such asymmetry creates persistent information imbalance and can induce systematic coordination failures if policies assume homogeneous sensing.

Several strands of work address heterogeneity by explicitly encoding capabilities into coordination mechanisms or by designing robust execution layers. Robust coordination under heterogeneous capabilities has been studied both in learning-free and learning-enabled settings, with an emphasis on ensuring that policy behavior remains valid when sensors are degraded or missing \cite{fu2022robust,lin2020multi}. Planning-based multi-robot coordination has similarly incorporated capability constraints, demonstrating that effective heterogeneous execution often requires structured interfaces between perception, planning, and control \cite{carreno2022planning}. DC-Ada is aligned with this interface-centric viewpoint: rather than learning an entirely new coordination strategy, it seeks to maintain a fixed shared policy by adapting the observation interface to compensate for capability mismatch.

\subsection{Communication-efficient coordination, mapping, and distributed estimation}
Communication constraints are a central theme in multi-robot systems, motivating approaches that reduce either message rate, payload size, or both. Classical coordination relies on consensus and distributed control principles to enable agreement and formation behaviors with limited information exchange \cite{choi2009consensus}. Distributed estimation and active perception frameworks further formalize how teams can maintain consistent beliefs while minimizing communication, particularly when sensing is decentralized and bandwidth is constrained \cite{atanasov2015decentralized}. A complementary line of work in distributed optimization studies when and how to communicate to preserve convergence guarantees, including event-triggered communication policies that send updates only when local changes exceed a threshold \cite{hu2021event,marcotte2020optimizing,patwardhan2024distributed}. Recent advances improve scalability by accelerating consensus dynamics and leveraging graph-structured designs \cite{mahato2023consensus,goarin2024graph}.

In multi-robot mapping and SLAM, communication efficiency is also a driving motivation: modern systems often separate lightweight front-ends from back-end optimization, and distribute computation and communication to reduce bandwidth while preserving map consistency \cite{cieslewski2017efficient,chang2021kimera,chen2022slam}. These approaches typically exchange structured information such as poses, keyframes, loop closures, or submaps. DC-Ada differs fundamentally in the communication primitive: it does not exchange maps, beliefs, or observations. Instead, it leverages only scalar team-level return feedback to adapt local observation interfaces, thereby operating in an extreme communication-minimal regime.

\subsection{Multi-agent reinforcement learning and learned communication}
MARL provides a spectrum of techniques for cooperative control under partial observability and decentralized execution. Independent learners and decentralized actor-critic methods have long been studied as baseline approaches, but they can suffer from non-stationarity and coordination instability in cooperative settings \cite{matignon2012independent}. Deep MARL methods such as MADDPG and cooperative deep RL architectures address coordination through centralized training and decentralized execution (CTDE), enabling agents to learn coordinated behaviors while maintaining decentralized policies at test time \cite{lowe2017multi,gupta2017cooperative}. Value factorization methods such as QMIX provide another CTDE mechanism that decomposes team value functions while preserving decentralized greedy execution \cite{rashid2018qmix}. Surveys emphasize that partial observability, non-stationarity, and credit assignment remain fundamental challenges in multi-agent learning \cite{nguyen2020deep}. Recent heterogeneous-agent MARL methods explicitly relax homogeneous parameter-sharing assumptions by training separate or sequentially updated policies for heterogeneous agents \cite{zhong2024harl}, and modular policy-network approaches similarly encode heterogeneity through learned policy components that specialize by agent, environment, or task \cite{lee2025modularheteromarl}. These methods are important comparators at the level of problem motivation, but they address a different operating regime: they modify the training procedure or policy architecture, whereas DC-Ada assumes the deployed policy is already trained and frozen and adapts only a lightweight observation interface at test time.

Learned communication extends MARL by enabling agents to exchange differentiable messages or learned symbols. Early work learned communication protocols end-to-end \cite{foerster2016learning}, while more recent methods explore message-passing networks and attention-based communication to scale coordination \cite{zhang2023commnetx,phan2023attention}. Additional lines incorporate structured coordination layers and decentralized communication networks \cite{Meng2023DC2Net,Yang2024TeamComm}. Language-augmented planning and communication has also emerged as a high-level coordination paradigm, where agents use natural language or language-like abstractions to coordinate navigation and planning \cite{Rajvanshi2025SayCoNav,Rana2023SayPlan,Zhao2024HAS}. DC-Ada targets a different deployment regime: it assumes persistent message exchange is undesirable or infeasible, and it does not learn a communication protocol. Instead, it adapts a local observation interface using only scalar reward, preserving a fixed shared policy and avoiding explicit message passing during execution.

\subsection{Sim-to-real transfer, data-centric robotics, and personalization}
Bridging the training-to-deployment gap is a longstanding theme in robotics. Meta-learning and model-based adaptation frameworks aim to enable rapid adaptation to new dynamics or sensing conditions \cite{finn2017model}. Modular sim-to-real transfer has emphasized decomposing perception and control into transferable components and adapting representations to cope with domain shift \cite{rusu2017sim2real}. Large-scale robotic learning systems further highlight that robustness often depends as much on data diversity and pipeline design as on algorithmic novelty \cite{kalashnikov2018scalable}. Data-centric AI perspectives similarly shift emphasis toward data curation, representation design, and adaptation mechanisms that maintain performance as conditions evolve \cite{zha2025data}. In distributed settings, federated and personalization approaches study how local updates can be performed under privacy and bandwidth constraints \cite{ho2022federated}. Training-time strategies such as prioritized replay and self-paced curricula improve robustness but typically assume access to full training infrastructure and repeated retraining cycles \cite{schaul2015prioritized,cheng2023prioritized,klink2020selfpaced}.

DC-Ada complements these directions by focusing on \emph{post-training} deploy-time adaptation under a strict interaction budget, with a minimal communication primitive. Rather than adapting the policy itself or retraining with new data, DC-Ada adapts the observation interface to reduce distribution shift for a frozen shared policy, aligning with data-centric intuition that improving input representations can yield practical robustness improvements.

\subsection{Derivative-free and bandit optimization for reward-only adaptation}
Optimizing a system using only scalar returns without gradients is the classical domain of stochastic approximation and simulation optimization \cite{kushner2003stochastic,borkar2008stochastic}. Finite-difference methods and simultaneous perturbation techniques (e.g., SPSA) provide canonical mechanisms for estimating descent directions under noisy objective evaluations \cite{spall1992multivariate}. Bandit and zeroth-order optimization extend these ideas with convergence and complexity analyses for random-direction estimators under Lipschitz or smoothness assumptions \cite{flaxman2005bandit,nesterov2017random,duchi2015optimal}. In reinforcement learning, evolution strategies and simple random search demonstrate that gradient-free optimization can scale to high-dimensional policies by exploiting shared randomness and low-dimensional broadcasts \cite{salimans2017es,mania2018simple}.

DC-Ada adopts a conservative, selection-based zeroth-order procedure, but differs from ES/ARS in two respects. First, it optimizes only a compact observation-transform module rather than full policy parameters, thereby reducing the effective dimension of the search space and improving stability for deployment. Second, it uses CRN-based candidate comparisons and an accept/reject margin to control noise-driven updates under a fixed interaction budget, aligning the update rule with a safety-oriented deployment perspective.

\subsection{Positioning relative to prior work}
\label{subsec:positioning}
The contribution of DC-Ada is intentionally integrative rather than a claim of a new optimizer or a new MARL training paradigm. Its distinction is the combination of four deployment constraints in a single evaluated pipeline: (i) the shared policy is pretrained and frozen, (ii) heterogeneity is handled through robot-specific observation-interface transforms rather than through policy retraining, (iii) adaptation uses only scalar team-return feedback rather than gradients, privileged state, maps, or raw observations, and (iv) all adaptation rollouts are charged against the same environment-step budget used by the baselines. This separates DC-Ada from training-time heterogeneous-agent MARL, learned-communication methods, policy fine-tuning, and ordinary observation normalization. The method is therefore best viewed as a deploy-time compatibility layer for fixed controllers under sensing mismatch, not as a replacement for richer training-time heterogeneous MARL when such retraining and validation are feasible.

\subsection{Biological motivation: reliability-weighted sensing}
A motivating analogy for DC-Ada is the notion of reliability-weighted sensory integration in biological systems. Human perception and sensorimotor control are often modeled as integrating multiple cues in proportion to their reliability, approaching statistically efficient fusion \cite{ernst2002optimal,vanbeers2002feeling}. Moreover, sensorimotor adaptation to changes in cue reliability has been studied in postural control and related settings, where the nervous system reweights sensory sources when their reliability changes \cite{peterka2002sensorimotor}. While DC-Ada is not a biologically faithful model, the analogy is conceptually useful: heterogeneous robots experiencing modality loss or sensor degradation can be viewed as facing changes in cue reliability, and an interface adapter can be interpreted as reshaping or reweighting cues presented to an otherwise fixed controller. This perspective motivates focusing adaptation on the observation interface rather than retraining the full coordination policy.

%=======================================================
\section{Problem Formulation}
\label{sec:problem_formulation}
We consider a team of $N$ robots interacting with an episodic environment over a finite horizon $T$. The interaction can be modeled as a partially observable Markov game with global state $s_t \in \mathcal{S}$, joint action $a_t=(a_t^1,\dots,a_t^N)\in\mathcal{A}_1\times\cdots\times\mathcal{A}_N$, transition dynamics $p(s_{t+1}\mid s_t,a_t)$, and a shared scalar team reward $r_t=r(s_t,a_t)\in\mathbb{R}$. Each robot executes \emph{decentralized} control: robot $i$ selects $a_t^i$ based only on its local observation history, not on the global state or other robots' raw measurements. We focus on the cooperative setting with a common objective, i.e., the team seeks to maximize expected undiscounted episodic return.

A shared policy $\pi_\theta$ is pretrained offline under nominal sensing (H0) and then \emph{frozen} at deployment. The frozen-policy assumption reflects deployment regimes in which retraining is undesirable or infeasible (e.g., due to limited onboard compute, lack of differentiability through the deployed control stack, or the need to preserve validated coordination behavior). The central challenge is that the observation distribution at deployment may differ systematically from the nominal training regime due to heterogeneous sensing across robots.

\subsection{Heterogeneity as an observation-model shift}
\label{subsec:problem_heterogeneity}
A heterogeneity regime $H$ specifies, for each robot $i$, a sensor configuration that determines which modalities are available and with what capability parameters (e.g., LiDAR ray count and range, camera field-of-view, and noise levels). We model heterogeneity as a shift in the observation model: for a fixed $H$, each robot receives observations generated by an $H$-dependent observation kernel
\begin{equation}
o_t^i \sim \mathcal{O}_{H}^{i}(\,\cdot \mid s_t\,), \qquad i=1,\dots,N,
\label{eq:obs_kernel}
\end{equation}
where $\mathcal{O}_{H}^{i}$ encodes both measurement statistics and measurement \emph{semantics} (e.g., missing modalities, degraded ranges, or altered resolution). To ensure compatibility of pretrained checkpoints and prevent confounds due to changing input dimensionality, we enforce a fixed-layout observation interface within each environment (Sec.~\ref{sec:obs_interface}). Concretely, for each environment, observations are represented as vectors $o_t^i(H)\in\mathbb{R}^{d}$ with a constant feature ordering and constant dimensionality $d$ across H0--H3; heterogeneity affects the \emph{distribution} of these vectors (e.g., segments corresponding to missing modalities are filled with neutral defaults and sensor parameters change the measurement statistics), but not their layout.

\subsection{Data-centric interface adaptation with a frozen shared policy}
\label{subsec:problem_adaptation}
Rather than adapting the shared policy parameters $\theta$, each robot $i$ maintains a lightweight, robot-specific transformation module $g_{\phi_i}:\mathbb{R}^{d}\rightarrow\mathbb{R}^{d}$ that is applied to its local observation before passing the result to the frozen shared policy. The deployed decentralized control law is therefore
\begin{equation}
a_t^i \sim \pi_\theta\!\big(g_{\phi_i}(o_t^i(H))\big), \qquad i=1,\dots,N,
\label{eq:decentralized_control}
\end{equation}
where $\pi_\theta$ is shared across robots and $g_{\phi_i}$ is robot-specific. Let $\phi \triangleq \{\phi_i\}_{i=1}^N$ denote the collection of transform parameters. For a fixed heterogeneity regime $H$, the objective is to select $\phi$ to maximize the expected episodic return under the induced closed-loop system:
\begin{equation}
\max_{\phi} \; J_H(\phi)
\;\triangleq\;
\mathbb{E}\!\left[\sum_{t=1}^{T} r(s_t,a_t)\right],
\label{eq:objective}
\end{equation}
where the expectation is taken over the initial state distribution, environment randomness, and the stochasticity of the policy (if any), under decentralized execution~\eqref{eq:decentralized_control} and $H$-dependent observations~\eqref{eq:obs_kernel}.

\subsection{Resource constraints: interaction budget and minimal communication}
\label{subsec:problem_constraints}
We study the regime of \emph{budgeted} deploy-time adaptation. Each run (environment $\times$ heterogeneity level $\times$ seed) is constrained by a strict interaction budget of $B$ joint environment steps (in our experiments, $B{=}200{,}000$). Budget is counted in environment steps (one step advances all robots jointly), and any additional rollouts used for adaptation must be charged against the same budget. In addition, we consider minimal-communication settings where robots do not exchange raw observations, intermediate features, maps, or gradients; at most, scalar team-level feedback (e.g., episodic return) may be broadcast to support coordination of the adaptation procedure. These constraints reflect practical multi-robot deployments in which bandwidth is limited, privacy or modularity constraints prevent sharing raw sensor streams, and gradient access to the deployed control policy is unavailable.

Under these conditions, the adaptation problem~\eqref{eq:objective} is inherently black-box: the team can observe scalar rollout returns, but does not have access to analytic gradients of $J_H(\phi)$ nor to privileged centralized state. DC-Ada (Sec.~\ref{sec:method}) is designed specifically for this constrained setting by adapting only the observation interface via conservative reward-only candidate selection.

% ============================================================
\section{Fixed Observation Interface and Heterogeneity Design}
\label{sec:obs_interface}

A recurring methodological pitfall in heterogeneous multi-robot evaluation is inadvertently changing the observation dimensionality or feature ordering across heterogeneity levels. Such changes can make pretrained checkpoints incompatible, force ad-hoc architectural modifications, or (worse) lead to silent inconsistencies in what the policy perceives at test time. These issues confound interpretation because performance differences can no longer be attributed solely to sensing heterogeneity; they may instead reflect changes in network input dimensionality, feature semantics, or preprocessing. To prevent this failure mode, we enforce a \emph{fixed-layout} observation interface for every robot \emph{within each environment}, and we hold this interface constant across H0--H3. This design ensures that (i) pretrained shared-policy checkpoints remain compatible across all heterogeneity levels, (ii) heterogeneity corresponds to a controlled \emph{distribution shift} in the observation stream rather than a change in the policy architecture, and (iii) adaptation methods are evaluated on a consistent input representation.

\paragraph{Fixed-layout observation vector.}
For each robot $i$, the observation is a concatenation of normalized kinematic features, environment-specific task features, and three modality segments:
\begin{equation}
o_i = [p(2),\ v(2),\ \text{extra}(k),\ \text{lidar}(16),\ \text{rgb}(32),\ \text{depth}(16)] \in \mathbb{R}^{d}.
\label{eq:fixed_obs_layout}
\end{equation}
The position $p$ and velocity $v$ are normalized to fixed ranges defined by the arena size and action limits. The task-dependent $\text{extra}(k)$ segment includes compact features that are shared across all robots in the same environment (e.g., task counters, normalized time remaining, or goal-related summaries). Importantly, while $k$ varies by environment, it is fixed across heterogeneity levels within that environment. As a result, the full observation dimension is environment-specific but constant over H0--H3: in our implementation $d{=}73$ for warehouse logistics, $d{=}71$ for search-and-rescue, and $d{=}70$ for collaborative mapping.

\paragraph{Handling missing modalities without changing dimensionality.}
Heterogeneity is induced by removing modalities or degrading sensor parameters for some robots. To preserve the fixed interface in~\eqref{eq:fixed_obs_layout}, missing modalities are \emph{not} removed from the observation vector; instead, their segments are filled with neutral constants chosen to represent ``no information'' rather than spurious structure. Specifically, the RGB segment (feature-like) is filled with zeros, while the LiDAR and depth segments (distance-like) are filled with ones (corresponding to normalized ``max range''). This choice has two practical benefits. First, it prevents missing modalities from introducing extreme out-of-range values that could destabilize the frozen shared policy. Second, it makes the absence of a modality \emph{detectable} to the transform module (via constant segments), allowing learned transforms to down-weight or ignore missing modalities while preserving dimensional compatibility.

\paragraph{Resampling LiDAR to a fixed segment length.}
To model heterogeneity in LiDAR resolution (e.g., varying ray counts), robots may produce native LiDAR readings with different lengths. We map these native readings to the fixed 16-dimensional LiDAR segment by linear resampling (interpolation) over the angular domain. This preserves coarse angular structure while ensuring that the policy always receives a fixed-length LiDAR representation. Critically, this choice avoids a common confound in heterogeneous studies: differences in network behavior caused by mismatched input sizes or inconsistent feature semantics.

Table~\ref{tab:obs_layout} summarizes the observation layout, normalization conventions, and default-fill policies used to guarantee checkpoint compatibility and controlled evaluation.

\input{tables/obs_layout_table.tex}

\subsection{Heterogeneity levels H0--H3}
We instantiate four heterogeneity regimes (Table~\ref{tab:heterogeneity}) spanning homogeneous sensing (H0) through severe heterogeneity (H3). The intent is to expose algorithms to a progressively more challenging sequence of distribution shifts while keeping the policy interface unchanged. The design follows established stress-testing practices for asymmetric or heterogeneous multi-agent systems, where capability mismatch is explicitly introduced to evaluate robustness and adaptation \cite{bettini2023heterogeneous,gao2023asymmetric,fu2022robust}.

H0 corresponds to fully homogeneous robots with identical modality availability and sensor parameters. H1 introduces mild heterogeneity by altering sensor parameters (e.g., range, field-of-view, or sampling density) and selectively removing one modality for a subset of robots, while still preserving substantial sensing overlap across the team. H2 induces moderate heterogeneity by mixing sensor suites across robots (e.g., different modality subsets and parameterizations), increasing the degree of partial observability mismatch. H3 is designed as a severe and deliberately adversarial regime: it includes single-modality robots (LiDAR-only, RGB-only, depth-only) and a degraded multimodal robot. This regime tests whether a method can maintain coordinated behavior when the team lacks a shared sensing basis and individual robots must operate with limited, modality-specific information.

Table~\ref{tab:heterogeneity} reports the full per-robot sensing configuration for each heterogeneity level, including modality availability and parameter degradations. Figure~\ref{fig:hetero_overview} provides qualitative context: it juxtaposes a conceptual severe-heterogeneity illustration, representative sensor-configuration motifs, and the top-down arena layout used for the warehouse-style environment in our simulator.

\input{tables/heterogeneity_table.tex}

\begin{figure*}[t]
\centering
\includegraphics[width=.9\linewidth]{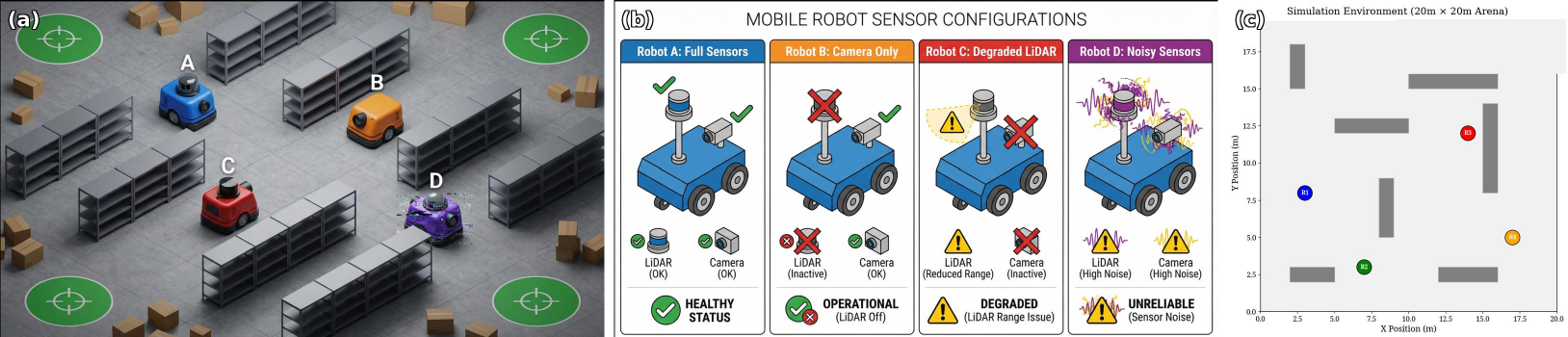}
\caption{Merged environment/heterogeneity illustration: (a) severe heterogeneity rendered scene (conceptual), (b) sensor configuration examples motivating heterogeneity, and (c) top-down arena layout for the warehouse-style environment used in our simulator.}
\label{fig:hetero_overview}
\end{figure*}

% ============================================================
\section{DC-Ada Method}
\label{sec:method}
DC-Ada is a deploy-time adaptation procedure for heterogeneous multi-robot teams that preserves a \emph{frozen} shared policy and adapts only a lightweight, robot-specific observation transform. The method operates under \emph{scalar reward feedback} and does not require policy gradients, centralized state, or observation sharing. Conceptually, DC-Ada treats the environment-and-policy stack as a black box: it proposes candidate interface perturbations, evaluates them via short rollouts under common random numbers (CRN), and applies a conservative accept/reject rule to update each robot's transform.

\subsection{Transformation parameterization}
\label{subsec:method_transform}
Each robot $i \in \{1,\ldots,N\}$ maintains a transform module $g_{\phi_i}:\mathbb{R}^{d}\rightarrow\mathbb{R}^{d}$ that maps the fixed-layout observation vector $o_{i,t}\in\mathbb{R}^{d}$ to a transformed observation $\hat{o}_{i,t}\in\mathbb{R}^{d}$ before inference by the frozen shared policy $\pi_\theta$. The transform is designed to be (i) low-capacity, to reduce overfitting and preserve stability, and (ii) identity-preserving at initialization, to ensure safe deployment prior to adaptation.

In our implementation, $g_{\phi_i}$ is a residual bottleneck MLP with latent dimension $d_z{=}32$ and a LayerNorm stabilizer:
\begin{align}
z &= \mathrm{LN}\!\left(\mathrm{ReLU}(W_e o + b_e)\right), \label{eq:adapter_enc}\\
h &= \mathrm{ReLU}(W_h z + b_h), \label{eq:adapter_hid}\\
\hat{o} &= o + (W_d h + b_d), \label{eq:adapter_dec}
\end{align}
where $o,\hat{o}\in\mathbb{R}^{d}$ and $(W_e,b_e,W_h,b_h,W_d,b_d)$ comprise $\phi_i$. The residual form in~\eqref{eq:adapter_dec} biases the transform toward small, localized corrections to the observation interface rather than wholesale remapping, and LayerNorm in~\eqref{eq:adapter_enc} reduces sensitivity to modality-dependent scale shifts. Crucially, initializing all transform parameters to zero yields an exact identity mapping ($\hat{o}=o$), which (i) prevents degradation from random initialization, (ii) makes gains attributable to adaptation rather than chance initialization, and (iii) supports stable incremental updates from a known baseline.

\subsection{Budgeted accept/reject zeroth-order adaptation}
\label{subsec:method_adaptation}
DC-Ada performs periodic adaptation rounds every $K$ \emph{episodes} under a fixed environment-step budget $B$ (Sec.~\ref{sec:experiments}). Adaptation is \emph{reward-only}: each decision uses only the scalar return of short rollouts and does not require gradients or privileged signals. The update is also \emph{coordinate-like}: when adapting robot $i$, only $\phi_i$ is perturbed while all other robots' transforms $\{\phi_j\}_{j\neq i}$ are held fixed. This isolates the effect of each robot's interface correction under a shared team objective.

\paragraph{Candidate generation and evaluation.}
At an adaptation round, DC-Ada evaluates one baseline and $M$ candidates for each robot $i$. Candidates are generated by isotropic Gaussian perturbations
\begin{equation}
\phi_i^{(m)} \;=\; \phi_i + \sigma\,\epsilon_m, \qquad \epsilon_m \sim \mathcal{N}(0,I),
\label{eq:perturb}
\end{equation}
where $\sigma>0$ controls exploration in transform-parameter space. Each candidate is evaluated using a truncated rollout of horizon $T_c<T$ (default $T_c/T=0.25$) to respect the fixed interaction budget. To reduce return variance in candidate comparisons, DC-Ada uses \emph{common random numbers} (CRN): baseline and candidate rollouts are executed with the same episode seed and environment initialization, so that differences in return are primarily attributable to the perturbation rather than stochastic rollouts (Sec.~\ref{sec:theory}).

Let $R_0$ denote the truncated return of the baseline rollout and $R_m$ the return of candidate $m$. DC-Ada selects the best candidate index $m^\star \leftarrow \arg\max_m R_m$ and applies a conservative accept/reject test:
\begin{equation}
\phi_i \leftarrow
\begin{cases}
\phi_i + \alpha\,\sigma\,\epsilon_{m^\star}, & \text{if } R_{m^\star} - R_0 > \tau,\\
\phi_i, & \text{otherwise},
\end{cases}
\label{eq:accept_reject}
\end{equation}
where $\alpha>0$ is a step size and $\tau\ge 0$ is an acceptance margin. The margin $\tau$ reduces false accepts when rollouts are noisy and also partially mitigates short-horizon mismatch (Sec.~\ref{sec:theory}) by requiring an improvement sufficiently large to be unlikely under noise alone.

\paragraph{Budget accounting and fairness.}
All baseline and candidate rollouts consume environment steps and are charged against the same budget $B$ used to evaluate baselines. Consequently, DC-Ada may execute fewer full-length episodes than non-adaptive methods, but the \emph{total environment interaction} remains matched across methods. This is essential for interpreting gains as robustness improvements rather than artifacts of unequal data usage.

Algorithm~\ref{alg:dcada} summarizes DC-Ada as implemented in our pipeline, including CRN candidate evaluation, truncated rollouts, coordinate-wise updates, and budgeted execution.

\begin{algorithm}[t] \scriptsize
\caption{DC-Ada (per run, budgeted)}
\label{alg:dcada}
\begin{algorithmic}[1]
\REQUIRE Frozen shared policy $\pi_\theta$, transforms $\{\phi_i\}_{i=1}^N$, budget $B$, update interval $K$, candidates $M$, noise scale $\sigma$, step size $\alpha$, accept margin $\tau$, candidate horizon $T_c$
\STATE Initialize steps used $\leftarrow 0$
\WHILE{steps used $< B$}
\STATE Execute one nominal episode using $a_{i,t}\sim \pi_\theta\!\big(g_{\phi_i}(o_{i,t})\big)$; increment steps used by episode length; log reward and success
\IF{episode index $\bmod K = 0$ \AND steps used $< B$}
\FOR{$i=1$ to $N$}
\STATE Baseline: run a CRN short rollout of length $T_c$ with current $\{\phi_j\}_{j=1}^N$; obtain $R_0$; increment steps used by $T_c$
\STATE Sample $\epsilon_1,\dots,\epsilon_M\sim \mathcal{N}(0,I)$
\FOR{$m=1$ to $M$}
\STATE Candidate: set $\phi_i^{(m)}=\phi_i+\sigma\epsilon_m$; run CRN short rollout of length $T_c$ with $\phi_i\!=\!\phi_i^{(m)}$ and $\phi_{j\neq i}$ fixed; obtain $R_m$; increment steps used by $T_c$
\ENDFOR
\STATE $m^\star \leftarrow \arg\max_m R_m$
\IF{$R_{m^\star}-R_0>\tau$}
\STATE $\phi_i \leftarrow \phi_i + \alpha\,\sigma\,\epsilon_{m^\star}$
\ENDIF
\ENDFOR
\ENDIF
\ENDWHILE
\end{algorithmic}
\end{algorithm}

\subsection{Communication model}
\label{subsec:method_comm}
DC-Ada assumes a minimal communication mechanism sufficient to provide scalar team-level feedback for candidate selection. In our protocol, communication is limited to broadcasting a per-rollout scalar return (and, optionally, an accepted candidate index), while no raw observations, maps, gradients, or intermediate features are exchanged. This design aligns with communication-minimal optimization paradigms and evolution-strategy-style update mechanisms that scale via shared randomness and low-dimensional broadcasts \cite{salimans2017es,mania2018simple}. In practical implementations, the scalar return can be computed by a central logger or aggregated through a low-rate broadcast; critically, the adaptation itself remains local to each robot's transform parameters and does not require message-passing policies or centralized belief-state inference.

% ============================================================
\section{Analysis and Design Rationale}
\label{sec:theory}
The objective in Sec.~\ref{sec:problem_formulation} has the structure of a \emph{black-box} stochastic optimization problem: the team can query scalar rollout returns, but does not access policy gradients, centralized belief states, or privileged global supervision. This section provides an analytical lens that (i) situates DC-Ada within zeroth-order and bandit optimization, (ii) motivates common-random-number (CRN) evaluation as a variance-reduction mechanism for candidate selection, (iii) clarifies what accept/reject selection implies (and does not imply) about expected improvement, and (iv) makes explicit the interaction and scalar-feedback costs incurred under a fixed environment-step budget.

\subsection{Black-box interface adaptation as zeroth-order optimization}
Let $\phi \in \mathbb{R}^d$ denote the concatenated parameters of all robot-specific observation transforms (i.e., $\phi \triangleq [\phi_1;\ldots;\phi_N]$), and let $\xi$ collect all sources of rollout randomness, including initial conditions, random placements, and any stochasticity in the environment dynamics. For a rollout of horizon $T$, define the realized return
\begin{equation}
F_T(\phi,\xi) \;\triangleq\; \sum_{t=1}^{T} r_t,
\end{equation}
and the objective
\begin{equation}
J_T(\phi) \;\triangleq\; \mathbb{E}_{\xi}\!\left[F_T(\phi,\xi)\right].
\end{equation}
The optimization of $J_T(\phi)$ using only function evaluations of $F_T(\phi,\xi)$ is the canonical setting of stochastic zeroth-order optimization and bandit feedback \cite{flaxman2005bandit,kushner2003stochastic,borkar2008stochastic,duchi2015optimal,nesterov2017random}. In this view, DC-Ada performs \emph{structured} black-box optimization: it restricts adaptation to a low-capacity interface module (the observation transform) while keeping the shared policy fixed, thereby reducing both the effective dimension $d$ and the risk of destabilizing previously learned coordination behavior.

\subsection{CRN evaluation as variance reduction for candidate comparisons}
DC-Ada relies on comparing a baseline rollout to one or more perturbed candidates and accepting an update only when a candidate empirically improves the return. Because rollout returns can exhibit substantial variance due to environment randomness, stabilizing these comparisons is essential for reliable accept/reject decisions. A standard variance-reduction technique in simulation optimization and evolution strategies is the use of \emph{common random numbers} (CRN), i.e., evaluating baseline and candidate under the same realization of $\xi$ \cite{salimans2017es}.

For two parameter vectors $\phi$ and $\phi'$, define the single-seed difference estimator
\begin{equation}
\widehat{\Delta}(\phi,\phi';\xi) \;\triangleq\; F_T(\phi',\xi) - F_T(\phi,\xi).
\end{equation}
Under CRN, both terms share the same $\xi$. By direct expansion,
\begin{align} 
\mathrm{Var}\!\left[\widehat{\Delta}(\phi,\phi';\xi)\right]
= &\mathrm{Var}[F_T(\phi',\xi)] + \mathrm{Var}[F_T(\phi,\xi)] - \nonumber\\ 
&2\,\mathrm{Cov}\!\left(F_T(\phi',\xi),F_T(\phi,\xi)\right).
\label{eq:crn_var}
\end{align}
In contrast, under independent-seed evaluation with $\xi$ and $\xi'$ independent,
\begin{equation} \scriptsize
\mathrm{Var}\!\left[F_T(\phi',\xi)-F_T(\phi,\xi')\right]
= \mathrm{Var}[F_T(\phi',\xi)] + \mathrm{Var}[F_T(\phi,\xi)].
\end{equation}
Therefore, whenever the covariance term in~\eqref{eq:crn_var} is positive (a typical regime when $\phi'$ is a small perturbation of $\phi$ and the induced trajectories remain similar), CRN strictly reduces the variance of the improvement estimate relative to independent-seed evaluation.

\begin{lemma}[CRN reduces the variance of rollout differences]
\label{lem:crn}
If $\mathrm{Cov}(F_T(\phi',\xi),F_T(\phi,\xi))>0$, then
\[
\mathrm{Var}\!\left[\widehat{\Delta}(\phi,\phi';\xi)\right]
<
\mathrm{Var}\!\left[F_T(\phi',\xi)-F_T(\phi,\xi')\right],
\]
where $\xi$ and $\xi'$ are independent.
\end{lemma}

Lemma~\ref{lem:crn} is directly relevant to DC-Ada because accept/reject selection depends on the sign and magnitude of measured improvements. Reducing the variance of $\widehat{\Delta}$ decreases the probability of both false accepts (accepting a non-improving candidate) and false rejects (rejecting an improving candidate) at a fixed evaluation budget.

\subsection{Accept/reject selection: reliability and what it guarantees}
Accept/reject guarantees improvement only on the \emph{sampled} evaluation outcomes and does not imply monotonic improvement of the expected objective $J_T(\phi)$. A standard reliability mechanism is to evaluate candidate improvements using $S\ge 1$ independent seeds and to accept a candidate only if its empirical mean improvement exceeds a margin $\tau\ge 0$. Define
\begin{equation}
\widehat{\Delta}_S(\phi,\phi') \;\triangleq\; \frac{1}{S}\sum_{s=1}^{S}\Big(F_T(\phi',\xi_s)-F_T(\phi,\xi_s)\Big),
\end{equation}
where $\{\xi_s\}_{s=1}^{S}$ are i.i.d. rollouts (CRN is applied within each pair). When returns are bounded, one can bound the probability of accepting a candidate whose expected improvement is non-positive.

\begin{proposition}[False-accept probability under bounded returns]
\label{prop:false_accept}
Assume $F_T(\phi,\xi)\in[-R,R]$ for all $\phi,\xi$ and let $\{\xi_s\}_{s=1}^{S}$ be i.i.d.
If a candidate $\phi'$ is accepted whenever $\widehat{\Delta}_S(\phi,\phi')>\tau$, then
\begin{equation} \scriptsize
\mathbb{P}\!\left(\, J_T(\phi')-J_T(\phi)\le 0 \ \wedge\ \widehat{\Delta}_S(\phi,\phi')>\tau \,\right)
\;\le\;
\exp\!\left(-\frac{S\,\tau^2}{8R^2}\right).
\label{eq:false_accept_bound}
\end{equation}
\end{proposition}
\begin{proof}
For each $s$, define $D_s \triangleq F_T(\phi',\xi_s)-F_T(\phi,\xi_s)\in[-2R,2R]$. If $J_T(\phi')-J_T(\phi)\le 0$ then $\mathbb{E}[D_s]\le 0$ and thus $\mathbb{E}[\widehat{\Delta}_S]\le 0$. Hoeffding's inequality yields
$\mathbb{P}(\widehat{\Delta}_S-\mathbb{E}[\widehat{\Delta}_S]>\tau)\le \exp(-S\tau^2/(8R^2))$.
\end{proof}

Proposition~\ref{prop:false_accept} formalizes a design trade-off: increasing $S$ or $\tau$ decreases false accepts but increases interaction cost (more evaluations per decision) or conservatism (fewer accepted updates). In the main sweep we use $S{=}1$ for efficiency and rely on (i) CRN to reduce variance in pairwise comparisons and (ii) multi-seed reporting to assess robustness of the overall method. The bound~\eqref{eq:false_accept_bound} nevertheless provides a principled rationale for increasing $S$ or $\tau$ in higher-stakes deployments.

\subsection{Why best-of-$M$ perturbations tends to move uphill}
DC-Ada evaluates $M$ random perturbations and selects the best candidate before applying an accept/reject test. This mechanism is closely related to random search and derivative-free optimization, where improvement is obtained by sampling directions that align with the (unknown) gradient \cite{nesterov2017random,duchi2015optimal,mania2018simple}. To make this connection explicit, suppose $J_T$ is differentiable and locally approximated by a first-order expansion. For a small perturbation $\delta=\sigma\epsilon$ with $\epsilon\sim\mathcal{N}(0,I)$,
\begin{equation}
J_T(\phi+\delta) \;\approx\; J_T(\phi) + \sigma\, \nabla J_T(\phi)^\top \epsilon.
\label{eq:first_order}
\end{equation}
The random variable $\nabla J_T(\phi)^\top \epsilon$ is Gaussian with mean $0$ and variance $\|\nabla J_T(\phi)\|_2^2$. Selecting the best among $M$ samples increases the expected maximum of these Gaussian projections by a factor that scales as $\mathbb{E}[\max_{m\le M} Z_m]$, where $Z_m\sim \mathcal{N}(0,1)$, and $\mathbb{E}[\max_{m\le M} Z_m]=\Theta(\sqrt{\log M})$. Consequently, in regimes where a first-order approximation is informative and $\sigma$ is sufficiently small, best-of-$M$ selection tends to produce candidate directions with positive alignment to $\nabla J_T(\phi)$.

A slightly more refined statement follows from standard smoothness arguments. If $J_T$ has $L$-Lipschitz gradient, then for any perturbation $\delta$,
\begin{equation}
J_T(\phi+\delta) \;\ge\; J_T(\phi) + \nabla J_T(\phi)^\top \delta - \frac{L}{2}\|\delta\|_2^2.
\label{eq:smooth_lower}
\end{equation}
Combining~\eqref{eq:smooth_lower} with $\delta=\sigma\epsilon$ suggests that best-of-$M$ selection increases the linear term in expectation, while the quadratic term penalizes overly large $\sigma$. This provides a theoretical motivation for DC-Ada's conservative perturbation scale and identity-preserving initialization: the method aims to remain within a locally smooth regime where sampled candidates are predictive of improvement.

\subsection{Truncated candidate rollouts and objective mismatch}
To respect a fixed interaction budget, DC-Ada evaluates baseline and candidates on truncated rollouts of horizon $T_c<T$ and uses these truncated returns for selection. Let $F_{T_c}(\phi,\xi)\triangleq \sum_{t=1}^{T_c} r_t$ and $J_{T_c}(\phi)=\mathbb{E}[F_{T_c}(\phi,\xi)]$. Candidate selection using $J_{T_c}$ is, in general, biased relative to the full-horizon objective $J_T$, and improvements under $T_c$ need not translate to improvements under $T$.

A standard way to quantify the mismatch is to assume bounded per-step rewards, $|r_t|\le r_{\max}$. Then for any $\phi,\xi$,
\begin{equation}
\big|F_T(\phi,\xi)-F_{T_c}(\phi,\xi)\big| \;\le\; (T-T_c)\,r_{\max}.
\label{eq:trunc_bound}
\end{equation}
Equation~\eqref{eq:trunc_bound} implies that if a candidate exhibits a sufficiently large empirical improvement under $T_c$ (and the accept margin $\tau$ is chosen conservatively), the risk of accepting a purely ``short-horizon'' improvement can be reduced. In practice, DC-Ada mitigates truncation mismatch through conservative update intervals, CRN-based candidate comparisons, and a fixed candidate horizon fraction, which together reduce the frequency of spurious acceptances driven by high-variance early returns.

\subsection{Interaction and communication complexity under a fixed budget}
Let $T$ denote the nominal episode horizon and $T_c$ the candidate horizon. Every $K$ episodes, DC-Ada performs (for each robot) one baseline rollout and $M$ candidate rollouts of length $T_c$ for candidate evaluation. Thus, one update round consumes
\begin{equation}
\Delta B_{\text{adapt}} \;=\; N\,(M+1)\,T_c
\label{eq:interaction_overhead}
\end{equation}
additional environment steps beyond the nominal execution steps. Under a fixed budget $B$, this overhead reduces the number of full episodes that can be executed, but does not increase total environment interaction. The design therefore induces a controlled trade-off: interaction is allocated either to longer nominal episodes or to short evaluation rollouts that improve the interface parameters.

Communication requirements remain minimal in the protocol considered here. Each rollout produces a scalar return that is broadcast (or equivalently logged and compared) to support candidate selection and accept/reject decisions. If scalar returns are represented as 64-bit floats, then each broadcast requires 8 bytes. Over one update round with $S$ evaluation seeds per decision, the communication volume scales as
\begin{equation}
\Delta C_{\text{adapt}} \;=\; 8 \, S \, N \,(M+1) \quad \text{bytes},
\end{equation}
which is negligible relative to message-passing approaches that exchange observations, maps, or gradients. This accounting formalizes the intended operating point of DC-Ada: reward-only, communication-light adaptation that treats the environment and the shared policy as black boxes while explicitly tracking the interaction cost incurred by candidate selection.

% ============================================================
\section{Baselines}
\label{sec:baselines}
We compare DC-Ada against four baselines implemented in the same codebase and evaluated under the same interaction budget ($B{=}200{,}000$ joint environment steps), heterogeneity levels (H0--H3), and random seeds. All methods share the same pretrained \emph{shared policy} parameters; the comparison therefore isolates the effect of \emph{test-time interface adaptation} (or the absence thereof) under heterogeneous sensing. For methods that perform additional rollouts for selection or learning (e.g., DC-Ada candidate evaluations or local fine-tuning updates), these rollouts are charged against the same environment-step budget to ensure fairness.

The baseline set is chosen to compare mechanisms that operate under the same frozen-policy deployment premise: no adaptation, statistics-only input stabilization, unselected transform perturbation, and gradient-based transform adaptation. Stronger training-time heterogeneous MARL baselines and learned-communication architectures are not direct drop-in baselines for this setting because they alter the policy architecture, training objective, or communication protocol before deployment. We therefore use them as related work and focus the empirical comparison on methods that can be applied after a shared policy has already been trained. To further clarify the role of DC-Ada's design choices and robustness, Sec.~\ref{sec:results_targeted} reports targeted H3 ablations and observation-stress tests under the same frozen-policy premise.

\paragraph{Shared Policy (no adaptation).}
This baseline executes the pretrained shared policy directly, without any test-time modification to the observation stream. It represents the natural deployment strategy when a single controller is trained under homogeneous sensing and then deployed unchanged across heterogeneous platforms. Because the shared policy is frozen and no additional computation is performed, this baseline provides a reference point for quantifying the magnitude of heterogeneity-induced degradation and the potential value of interface adaptation.

\paragraph{Observation Normalization.}
This baseline applies online mean/variance normalization to the fixed-layout observation vector prior to policy evaluation. The normalization statistics are updated during evaluation using a running estimator and are reset at the start of each run (environment $\times$ heterogeneity level $\times$ seed) to avoid cross-run leakage. Observation normalization does not modify the policy parameters and does not introduce a learned adapter; it instead targets a common failure mode under heterogeneity in which modality loss or sensor-parameter shifts induce large changes in the scale and distribution of input features. This baseline therefore tests whether distribution shift can be mitigated by stabilizing feature magnitudes alone, without introducing an additional trainable interface.

\paragraph{Random Perturbation.}
This baseline uses the same per-robot transform module as DC-Ada but removes the key selection mechanism. At a fixed adaptation interval, each robot perturbs its transform parameters with additive Gaussian noise and continues execution without evaluating candidates or applying an accept/reject rule. In other words, this baseline isolates the effect of ``trying different transforms'' while controlling for architectural capacity and update frequency. Any performance gains relative to the shared policy may be attributed to incidental exploration in transform space, whereas gains of DC-Ada beyond this baseline can be attributed to conservative best-candidate selection under reward feedback.

\paragraph{Local Fine-Tuning.}
This baseline performs gradient-based updates of the per-robot transform parameters while keeping the shared policy frozen. Concretely, each robot maintains its own transform parameters and updates them using a policy-gradient objective computed from the frozen shared policy's log-likelihood and the observed episode return (or truncated-rollout return, depending on the configured update horizon). This baseline is intentionally strong: when gradients are stable and on-device compute is available, direct optimization of transform parameters can exploit informative shaped rewards more efficiently than reward-only candidate search. At the same time, this baseline remains comparable to DC-Ada in that it adapts only the interface module (not the shared policy) and is evaluated under the same interaction budget, so any benefit reflects the advantage of gradient information rather than additional data.

\paragraph{Hyperparameters and matching constraints.}
All baselines operate on the same fixed-dimensional observation layout and are evaluated with identical episode horizons and termination criteria. Hyperparameters (e.g., adaptation intervals, perturbation scales, fine-tuning learning rates, and update horizons) are reported in Table~\ref{tab:hyperparams}. We emphasize that DC-Ada and random perturbation use only scalar reward feedback for their update decisions, while local fine-tuning leverages gradients through the frozen policy to update the transform parameters; this distinction is central to interpreting the empirical trade-offs.

\input{tables/hyperparams_table_access.tex}

% ============================================================
\section{Experimental Setup}
\label{sec:experiments}

\subsection{Simulator, robot dynamics, and observation interface}
All experiments are conducted in a lightweight, deterministic, two-dimensional multi-robot simulator implemented in pure NumPy. The simulator is designed to support controlled heterogeneity studies and reproducible large sweeps across methods, seeds, and sensing configurations. Each episode evolves for at most $T{=}500$ discrete time steps. Robots follow simple velocity control with a fixed integration step ($\Delta t{=}0.1$) and bounded actions. At every step, each robot selects a two-dimensional action $a_t \in [-1,1]^2$ that is mapped to a planar velocity command and applied with a fixed speed scale; positions are clipped to remain inside the world boundaries. Obstacle collisions are handled by rejecting the position update when the proposed motion would intersect an obstacle. Robot--robot collisions are recorded as events when inter-robot distance falls below a fixed radius, and these events contribute to the task-specific reward where applicable.

A key requirement for evaluating heterogeneous sensing under a frozen shared policy is to maintain a \emph{fixed-dimensional observation interface} across robots and heterogeneity levels. To this end, the simulator enforces a constant feature layout for each environment of the form
\[ \scriptsize
[\text{pos}(2),\ \text{vel}(2),\ \text{extra}(k),\ \text{LiDAR}(16),\ \text{RGB}(32),\ \text{Depth}(16)] 
\]
where $\text{extra}(k)$ denotes environment-specific task features and the modality segments have fixed sizes. Position is normalized by the world size, and velocities are recorded in simulator units. When a robot lacks a modality, the corresponding segment is filled with a neutral default: the RGB segment is set to zeros (feature-like), while the LiDAR and Depth segments are set to ones (distance-like, corresponding to ``max range''). Robots may have different native LiDAR ray counts depending on heterogeneity level; native LiDAR readings are therefore linearly resampled to the fixed 16-dimensional LiDAR segment to preserve a consistent feature layout and ensure checkpoint compatibility across heterogeneity.

The simulator provides three sensing modalities with simple but explicit generative models. LiDAR readings are computed by ray casting against circular obstacles and world boundaries, and (optionally) against task targets (packages/victims/points of interest) modeled as small circles; distances are normalized by the robot-specific LiDAR range. RGB observations are represented as a compact 32-dimensional feature vector encoding up to eight nearby targets within the robot's camera field of view; each target contributes a small set of geometric features (normalized distance and relative bearing components). Depth observations are represented by a 16-dimensional angular sweep over a limited field of view, with distances normalized by a fixed depth range. These simplified sensing models are sufficient to induce meaningful distribution shifts under heterogeneity while remaining transparent and computationally efficient.

\subsection{Heterogeneity protocol}
Heterogeneity is induced by varying each robot's sensor suite and sensor parameters according to four levels (H0--H3). H0 corresponds to homogeneous robots with full sensing. H1 introduces mild variations in sensor parameters and partial modality availability. H2 uses mixed sensor suites with different modality combinations and parameter settings. H3 represents severe heterogeneity, including robots that possess only a single modality (LiDAR-only, RGB-only, or Depth-only) as well as a robot with a severely degraded LiDAR configuration. Across all levels, the fixed observation interface described above is preserved, ensuring that adaptation methods operate on a constant-dimensional input and that changes in performance can be attributed to sensing differences rather than to incompatible network interfaces. The exact per-robot sensor configurations and parameters are reported in the paper's heterogeneity overview figure and the corresponding setup tables.

\subsection{Environments and tasks}
We evaluate DC-Ada and baselines on three domains that capture distinct multi-robot coordination demands under heterogeneous sensing.

\paragraph{Warehouse logistics.}
Warehouse logistics models navigation in cluttered layouts with discrete pickup-and-delivery events, motivated by automated fulfillment and heterogeneous warehouse-robot operations \cite{Kang02122025}. The environment is a bounded $20{\times}20$ workspace with circular obstacles representing shelving units placed at fixed locations. Packages are sampled at random free-space locations and two drop-off zones are positioned near opposing sides of the arena. Robots must navigate to packages, pick them up when within a pickup radius, and deliver them to a drop-off zone. The simulator tracks the delivered count, pickup events, and collision events. To preserve comparability across methods, the same initialization protocol and random seed control are used across all runs, and robots are spawned in distinct quadrants with bounded random perturbations.

\paragraph{Search-and-rescue.}
Search-and-rescue models exploration under time pressure with victim detection and retrieval, motivated by disaster robotics \cite{murphy2019disaster}. The environment is a $30{\times}30$ workspace with randomly placed debris obstacles. Victims are placed at random collision-free locations and are associated with a health variable that decreases over time, incentivizing rapid discovery and rescue. A victim is marked as found when a robot enters a detection radius and is marked as rescued when a robot enters a smaller rescue radius. The environment returns shaped rewards for finding and rescuing victims and includes a coverage-style exploration term to encourage broad search.

\paragraph{Collaborative mapping.}
Collaborative mapping models exploration and coverage, motivated by multi-robot mapping and SLAM front-ends \cite{cieslewski2017efficient,chang2021kimera,chen2022slam}. The environment is a $20{\times}20$ workspace discretized into a fixed-resolution grid. Robots mark grid cells within a fixed sensing radius as explored, producing a global exploration map and an episode-level coverage ratio. The reward includes a coverage term, a bonus for newly explored area, and a spread term encouraging the team to disperse spatially. This domain is particularly sensitive to perception and local sensing cues, and thus provides a complementary setting to discrete-event tasks for studying the benefits of observation-interface adaptation.

\subsection{Success criteria and progress metrics}
Each episode runs for at most $T{=}500$ steps and terminates early upon satisfying a task-specific success condition. Success is defined by explicit thresholds on task completion variables: deliveries in warehouse logistics, rescues in search-and-rescue, and coverage in collaborative mapping. The exact thresholds and termination conditions are reported in Table~\ref{tab:setup}. These thresholds are intentionally non-trivial operational targets rather than one-event minima; accordingly, success can become a tail event under strong heterogeneity, which is why we also report continuous progress variables and H3 threshold-sensitivity curves in addition to binary success. Because thresholded completion can be stringent in heterogeneous settings, we also record continuous progress metrics that quantify how close an episode is to completion, namely delivery ratio, rescue ratio, and coverage. Reward functions, including event-based and shaping components, are specified in Table~\ref{tab:reward}. The combination of reward, success, and progress metrics enables evaluation that is informative both in dense-feedback regimes and in sparse-completion regimes.

\input{tables/setup_table_revised.tex}
\input{tables/reward_table.tex}

\subsection{Budgets, seeds, and fairness protocol}
All methods in the main sweep are evaluated under the same environment-step budget of $B{=}200{,}000$ joint environment steps per configuration (environment $\times$ heterogeneity level $\times$ seed). Budget is counted in environment steps (one step advances all robots jointly), rather than per-robot steps. Candidate rollouts used by DC-Ada and gradient updates used by local fine-tuning are charged against the same interaction budget to ensure that comparisons reflect a true robustness--interaction trade-off rather than unequal data usage. Each main-sweep configuration is evaluated with five random seeds $\{0,1,2,3,4\}$, and results are reported as mean $\pm$ standard deviation across seeds. The full sweep comprises $3$ environments $\times\ 5$ methods $\times\ 4$ heterogeneity levels $\times\ 5$ seeds, yielding $300$ runs. Tables and figures that cover H0--H3 correspond to this final $200{,}000$-step sweep; Sec.~\ref{sec:results_targeted} reports additional H3-only follow-up studies that use the same per-run budget and are explicitly labeled as targeted checks.

To reduce confounding variation across methods, the simulator is deterministically reseeded at the start of each episode using a seed derived from the run seed and the episode index. Policies are evaluated deterministically (mean action) unless otherwise noted, ensuring that observed differences are attributable to sensing and adaptation rather than to action-sampling variance. For methods that compare candidate perturbations (DC-Ada), the experiment runner supports common-random-number evaluation by reusing the same episode seed across candidate rollouts, which reduces selection variance when deciding whether to accept an update.

\subsection{Pretraining and reproducibility}
For each environment, we pretrain a shared policy under homogeneous sensing (H0) and then freeze it for all subsequent evaluations and adaptations. This protocol isolates the effect of observation-interface adaptation from policy training and ensures that all methods begin from the same coordination behavior. Pretraining uses an advantage actor--critic procedure with a lightweight value network and entropy regularization; the policy is a tanh-squashed Gaussian with a two-layer MLP backbone ($256$ hidden units per layer) and separate heads for the mean and log-standard-deviation parameters. During pretraining, the environment returns a scalar \emph{team} reward; policy gradients are computed by aggregating per-robot log-probabilities and values via averaging, consistent with a shared-policy formulation
\footnote{
All code, configuration files, and scripts to reproduce the full sweep and generate plots/tables are provided at \url{https://github.com/alqithami/DC-ADA}. The reported main sweep corresponds to the final strong configuration (\texttt{configs/strong.yaml}) and the $200{,}000$-step results file used to generate the tables and figures. The experiment runner emits structured results with run metadata (timestamp, platform, Python and library versions), per-episode logs, and aggregated statistics, enabling independent verification of the reported figures and tables.}.

\subsection{Runtime and scalar-feedback accounting}
We report per-run wall-clock time, throughput (steps/s), and episode counts as measured by the experiment runner, and we record a minimal scalar-feedback proxy. Specifically, each rollout contributes a single scalar team-return value (8 bytes) used for evaluation or logging, while no observations, maps, or gradients are exchanged. This accounting yields a conservative low-bandwidth proxy for reward-only adaptation and enables direct comparison to approaches that would require richer message passing. Table~\ref{tab:overhead} summarizes runtime and scalar-feedback statistics for the reported sweep. The reported runs were executed on an Apple M4 Max system with 36\,GB unified memory, and the results metadata records the software stack used to produce the measurements.

\input{tables/overhead_table_access_revised.tex}

% ============================================================

% ============================================================
\section{Results}
\label{sec:results}
All methods are evaluated under a matched interaction budget of $B{=}200{,}000$ joint environment steps, across heterogeneity levels H0--H3 and five random seeds. The evaluation reports two complementary views of performance. The first is the \emph{mean episode reward}, which aggregates shaped progress signals over an episode and is therefore sensitive to incremental improvements in navigation, search, exploration, and task execution. The second is the \emph{success rate}, defined as the fraction of episodes satisfying a task-completion predicate (Table~\ref{tab:setup}). Because thresholded completion can be sparse under severe heterogeneity, we additionally report continuous progress metrics (deliveries/rescues/coverage) and success-threshold sensitivity under H3 (Sec.~\ref{sec:results_h3}), which provides a distributional view of completion difficulty beyond a single fixed threshold. A central empirical takeaway is that no single mitigation strategy dominates across all domains and metrics; the results are best interpreted as a comparison of operating points under shared interaction constraints. In particular, we do not interpret the evaluation as evidence that DC-Ada is a universal top performer in shaped reward; rather, it is a reward-only, communication-light alternative whose advantage appears most clearly in completion-sensitive mapping regimes.

\subsection{Aggregate reward performance and overall trends}
\label{sec:results_aggregate}
Table~\ref{tab:results} summarizes mean episode reward (mean $\pm$ standard deviation across seeds) for each environment, method, and heterogeneity level. Figures~\ref{fig:scaling} and~\ref{fig:performance} visualize the same results to highlight trends as heterogeneity increases.

A consistent pattern across domains is that increasing heterogeneity alters both the statistics and semantics of local sensing (e.g., modality loss, range reduction, ray-count reduction, and noise), inducing a distribution shift relative to the shared-policy pretraining regime. Under such shifts, a frozen shared policy can degrade, and the relative effectiveness of interface-level mitigation depends on the task structure and on the informativeness of the shaped reward.

In the warehouse domain, observation normalization yields the most robust reward profile across H0--H3, remaining between $268.9$ and $270.8$ reward units, whereas the frozen shared policy remains between $251.6$ and $254.7$ (Table~\ref{tab:results}). This improvement indicates that correcting input statistics can partially compensate for sensor mismatch in dense navigation-and-delivery settings. In collaborative mapping, mean reward is comparatively stable across heterogeneity levels for the strongest baselines: the frozen shared policy achieves approximately $4.77\times10^3$ reward units across H0--H3, with local fine-tuning close behind; DC-Ada remains within the same order and stays within roughly $1\%$ of these values despite relying on reward-only updates. In search-and-rescue, the ranking varies with heterogeneity: observation normalization is strongest under H0 and H2 (approximately $17.7$ and $17.2$), while the frozen shared policy is strongest under H1 and H3 (approximately $16.5$ and $16.1$). Local fine-tuning yields stable but lower shaped reward in this domain under the chosen schedule, underscoring that optimizing a front-end transform by gradients does not necessarily dominate strong non-adaptive baselines when the shared coordination policy is held fixed.

\input{tables/results_table_access.tex}

\subsection{Scaling with heterogeneity}
\label{sec:results_scaling}
Figure~\ref{fig:scaling} reports reward as a function of heterogeneity level for each domain. The warehouse task shows a mild degradation for most methods as heterogeneity increases. Observation normalization maintains the highest reward throughout, indicating that in this domain the dominant failure mode induced by heterogeneity is consistent with distributional shifts that can be partially mitigated by stabilizing the input scale. DC-Ada is roughly comparable to the frozen shared-policy baseline in warehouse reward, but it does not close the gap to observation normalization under the present shaping and budget; this suggests that, for this discrete pickup-and-delivery task, scalar-return interface search is less effective than direct input-statistics stabilization.

In search-and-rescue, the effect of heterogeneity is more pronounced and method-dependent. The shared policy peaks at H1 and then declines toward H3, while local fine-tuning remains robust in both reward and success (Sec.~\ref{sec:results_success}). Observation normalization exhibits a non-monotonic pattern: it underperforms under H1 but becomes the strongest reward baseline under H3. This non-monotonicity is consistent with heterogeneity changing which modalities dominate the effective observation and highlights that ``robustness'' is not purely a function of increasing heterogeneity, but of which specific sensing profiles occur at each level.

In collaborative mapping, reward varies less dramatically across H0--H3 in absolute terms, but the method ordering is stable. The frozen shared policy provides the highest reward at every heterogeneity level, with local fine-tuning consistently close behind; DC-Ada remains competitive in reward while delivering the strongest completion performance. Observation normalization is consistently weaker in mapping and collapses under H3 in success (Sec.~\ref{sec:results_success}), indicating that normalizing inputs alone does not necessarily preserve the spatial exploration and coverage behaviors required for completion under the mapping threshold.

\begin{figure*}[ht]
\centering
\includegraphics[width=0.30\textwidth]{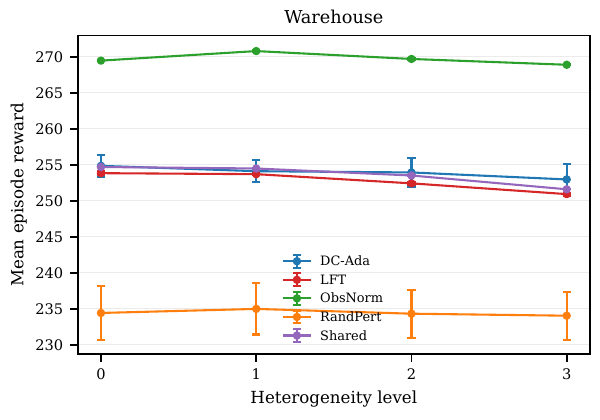}
\includegraphics[width=0.30\textwidth]{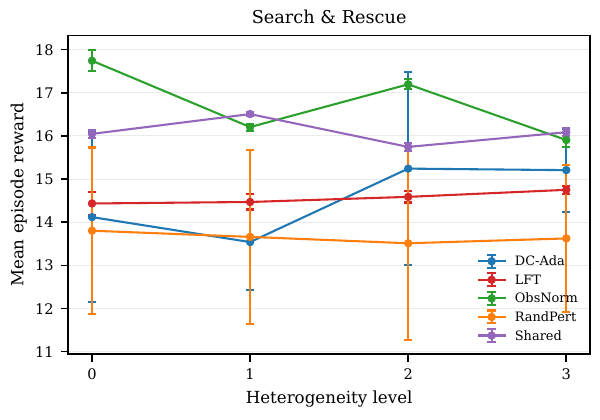}
\includegraphics[width=0.30\textwidth]{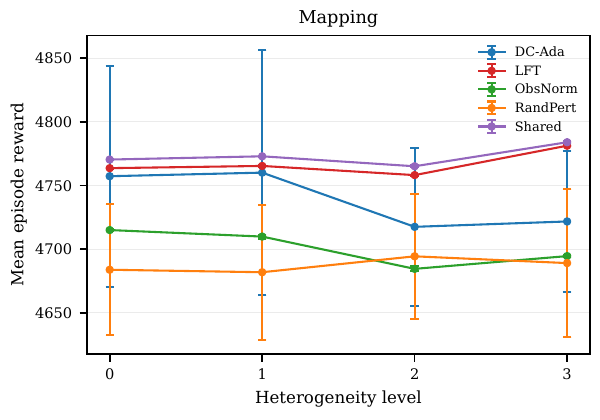}
\caption{Scaling with heterogeneity level (H0--H3) for the three tasks. Curves show mean reward under a fixed interaction budget ($B{=}200{,}000$ joint environment steps). A flatter slope indicates greater robustness to sensor mismatch for that task.}
\label{fig:scaling}
\end{figure*}

\begin{figure*}[ht]
\centering
\includegraphics[width=0.30\textwidth]{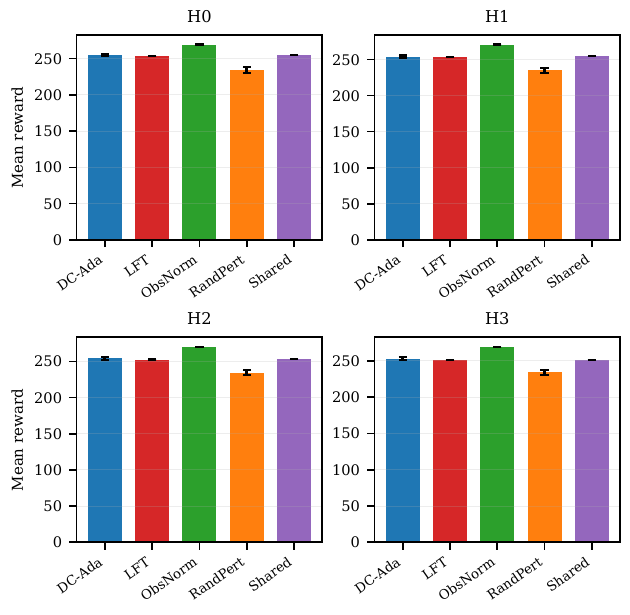}
\includegraphics[width=0.30\textwidth]{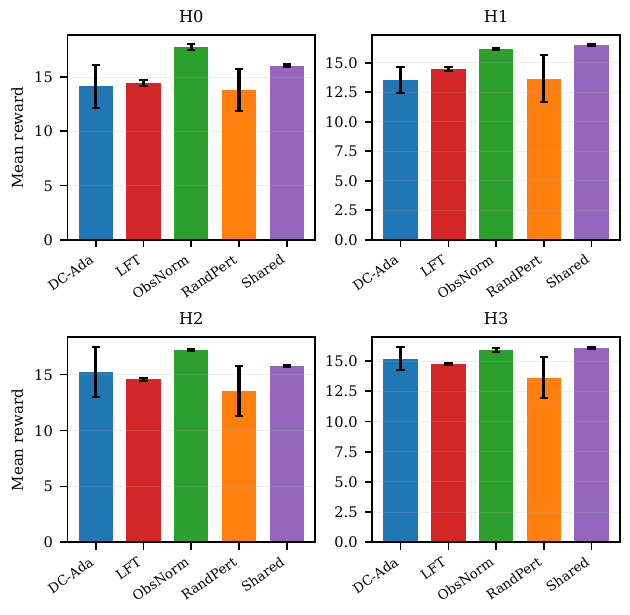}
\includegraphics[width=0.30\textwidth]{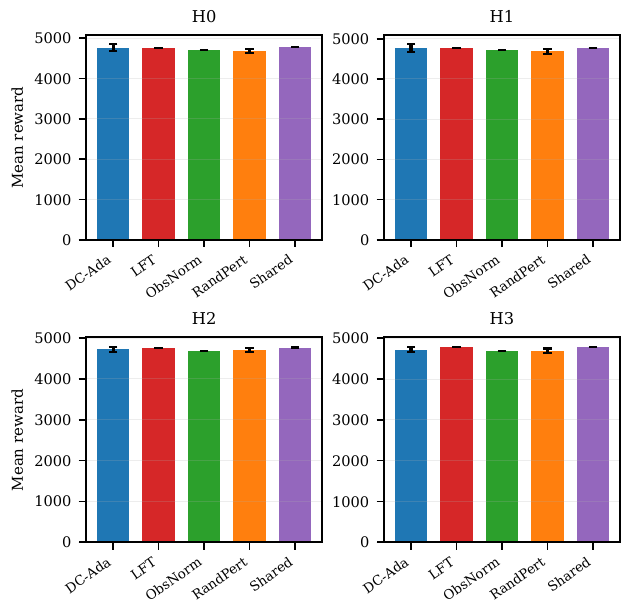}
\caption{Compact per-domain reward summary by method and heterogeneity level (H0--H3). Bars are averaged over five seeds under the same matched budget of $B{=}200{,}000$ joint environment steps. This figure complements Fig.~\ref{fig:scaling} by making within-level method ranking explicit.}
\label{fig:performance}
\end{figure*}

\subsection{Success rates and task completion}
\label{sec:results_success}
While reward captures shaped progress, success rate reflects whether an episode meets an explicit completion threshold (Table~\ref{tab:setup}). Figure~\ref{fig:success_rate} reports success rates across methods and heterogeneity levels for all three domains.

Three observations are immediate. First, completion is highly task-dependent: warehouse completion is rare for all methods (on the order of $1\%$--$2\%$ under the two-delivery threshold even with the extended budget), search-and-rescue exhibits moderate completion (roughly $10\%$--$15\%$ for the two-rescue threshold), and mapping exhibits low-to-moderate completion but clear method separation. Second, the method that maximizes reward does not necessarily maximize success. In mapping, the frozen shared policy attains the highest reward across heterogeneity levels, yet DC-Ada attains the highest success rates by a large margin. Third, heterogeneity affects success differently by domain: warehouse success remains uniformly low across H0--H3, search-and-rescue success varies modestly with heterogeneity and method, and mapping success under DC-Ada increases with heterogeneity, indicating that interface-level adaptation is particularly beneficial when completion depends on sustained coverage rather than a small number of discrete events.

\begin{figure*}[ht]
\centering
\includegraphics[width=0.30\textwidth]{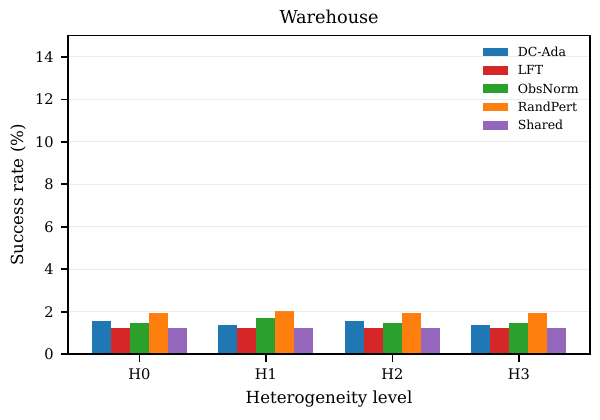}
\includegraphics[width=0.30\textwidth]{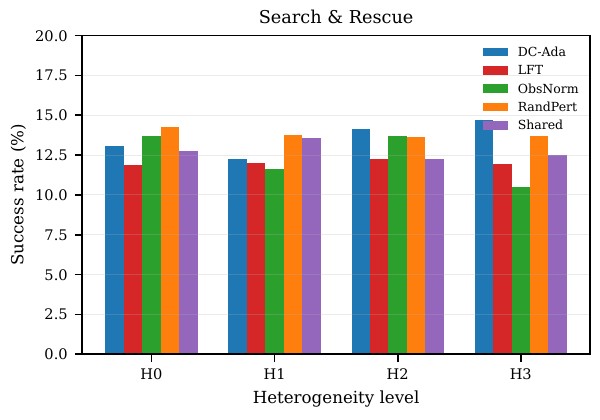}
\includegraphics[width=0.30\textwidth]{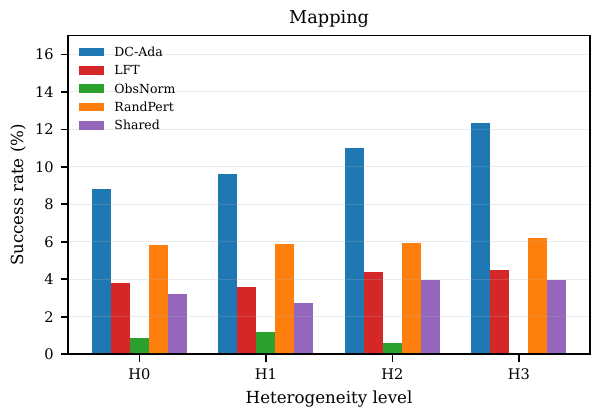}
\caption{Success rate by method and heterogeneity level (H0--H3) using the thresholds in Table~\ref{tab:setup}. Success is a binary completion metric; Sec.~\ref{sec:results_h3} reports continuous progress metrics and threshold sensitivity under severe heterogeneity.}
\label{fig:success_rate}
\end{figure*}

\subsection{Per-environment results and interpretation}
\label{sec:results_envwise}
\textbf{Warehouse.}
Under the default completion criterion (deliver at least two packages), warehouse success remains rare across all methods and heterogeneity levels, ranging from approximately $1.2\%$ to $2.1\%$ across the sweep (Fig.~\ref{fig:success_rate}). Random perturbation attains the highest completion in this regime (approximately $2.0\%$ under H0/H3), with observation normalization and DC-Ada close behind. Reward reveals clearer separation: observation normalization achieves the highest and most stable reward (approximately $269$--$271$ across H0--H3), improving over the frozen shared policy (approximately $252$--$255$) by roughly $14$--$18$ reward units (Table~\ref{tab:results}). The persistence of low success despite higher reward indicates that completion is dominated by rare sequences of discrete pickup-and-delivery events within the horizon. This interpretation is reinforced by the H3 threshold sensitivity analysis (Sec.~\ref{sec:results_h3}), which shows that delivering at least one package occurs in roughly $10\%$--$15\%$ of episodes, while completing two deliveries occurs in only $1.2\%$--$2.0\%$ of episodes, and three deliveries are not observed in the sweep.

\textbf{Search-and-Rescue.}
Search-and-rescue exhibits moderate completion rates and a clearer link between continuous progress and success than warehouse logistics. Under the two-rescue threshold, success typically falls in the $10\%$--$15\%$ range across H0--H3 (Fig.~\ref{fig:success_rate}). Under severe heterogeneity (H3), DC-Ada attains the highest success ($14.7\%$), compared to $13.7\%$ for random perturbation and $12.5\%$ for the frozen shared policy. In terms of shaped reward, however, observation normalization and the frozen shared policy remain strongest (Table~\ref{tab:results}), indicating that reward and completion capture distinct aspects of behavior under a fixed horizon. Continuous progress under H3 further clarifies the regime: the mean rescue ratio is approximately $0.29$--$0.31$, implying roughly $0.58$--$0.62$ rescues per episode against a target of two. Threshold sensitivity (Sec.~\ref{sec:results_h3}) shows that achieving at least one rescue occurs in roughly $46\%$--$49\%$ of H3 episodes, while achieving two rescues requires shifting the tail of the distribution and achieving three rescues is extremely rare. These results suggest that the success threshold is stringent relative to the horizon and that methods differ primarily in their ability to convert episodes with partial progress into episodes that complete the second discrete event.

\textbf{Collaborative Mapping.}
Collaborative mapping is the domain in which DC-Ada provides the clearest completion advantage. Across H0--H3, DC-Ada achieves the highest success rate for every heterogeneity level, increasing from $8.8\%$ (H0) to $12.4\%$ (H3), compared to $2.7\%$--$4.0\%$ for the frozen shared policy and $3.6\%$--$4.5\%$ for local fine-tuning (Fig.~\ref{fig:success_rate}). Under H3, DC-Ada attains $12.35\%$ success versus $6.21\%$ for random perturbation and $3.97\%$ for the shared policy, demonstrating that reward-only interface adaptation can substantially shift completion even when the coordination policy is fixed. In terms of continuous progress, DC-Ada also yields the highest mean coverage under H3 ($0.675$), compared to approximately $0.658$--$0.659$ for the strongest baselines, while the success threshold is $0.75$ (Table~\ref{tab:progress_h3}). The threshold sensitivity curves (Sec.~\ref{sec:results_h3}) show that DC-Ada's advantage is not confined to a single threshold: under H3, DC-Ada improves completion probability across a range of coverage thresholds (e.g., at $\tau{=}0.65$, $68.8\%$ for DC-Ada vs. $57.2\%$ for the shared policy; at $\tau{=}0.70$, $31.8\%$ vs. $22.8\%$). Notably, the frozen shared policy achieves the highest reward in mapping across heterogeneity levels (Table~\ref{tab:results}), yet it does not maximize thresholded success, highlighting that shaped reward and the completion predicate can emphasize different aspects of performance under fixed horizons.

To make this central mapping interpretation more explicit, Table~\ref{tab:mapping_contrasts} reports selected seed-paired contrasts for H3. The reward contrast favors the shared policy on average but has a wider interval that overlaps zero, whereas the success-rate and coverage contrasts favor DC-Ada with intervals entirely above zero. This compact summary reinforces the main empirical point: in collaborative mapping, DC-Ada's advantage is completion-oriented rather than reward-dominant.

\input{tables/mapping_contrasts_table.tex}

\subsection{Heatmap diagnostics across methods and heterogeneity}
\label{sec:results_heatmaps}
Figures~\ref{fig:heatmap_reward} and~\ref{fig:heatmap_success} provide a compact cross-method view of robustness across heterogeneity levels. In warehouse, the reward heatmap shows observation normalization as the most stable and highest-performing row, consistent with the scaling curves in Fig.~\ref{fig:scaling}. In mapping, reward remains comparatively flat across heterogeneity for the strongest baselines (shared policy and local fine-tuning), whereas success separates methods sharply: DC-Ada forms the strongest row across H0--H3, while observation normalization collapses to near-zero completion under H3. In search-and-rescue, both reward and success exhibit moderate variation with heterogeneity and method, with DC-Ada improving completion in the more challenging regimes. Overall, the heatmaps emphasize that robustness conclusions depend on both shaped reward and thresholded completion, and motivate the H3 progress and threshold-sensitivity analyses in Sec.~\ref{sec:results_h3}.

\begin{figure*}[ht]
\centering
\includegraphics[width=0.30\textwidth]{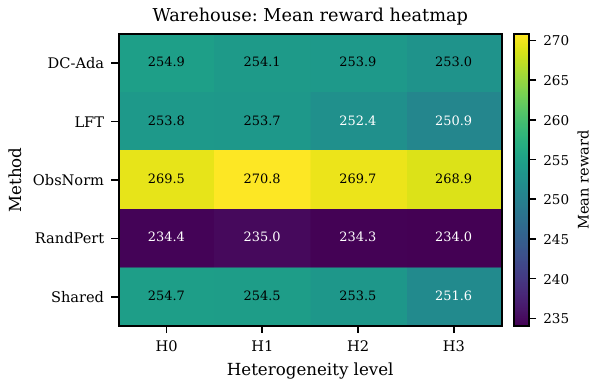}
\includegraphics[width=0.30\textwidth]{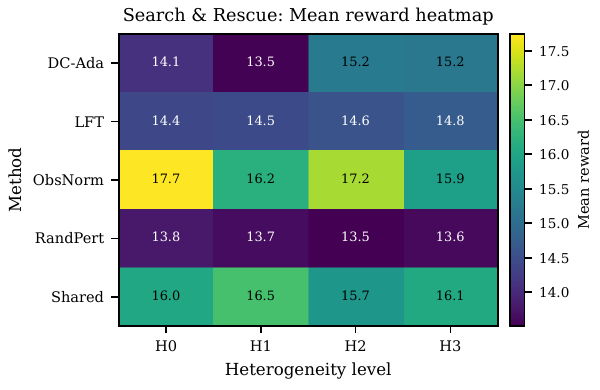}
\includegraphics[width=0.30\textwidth]{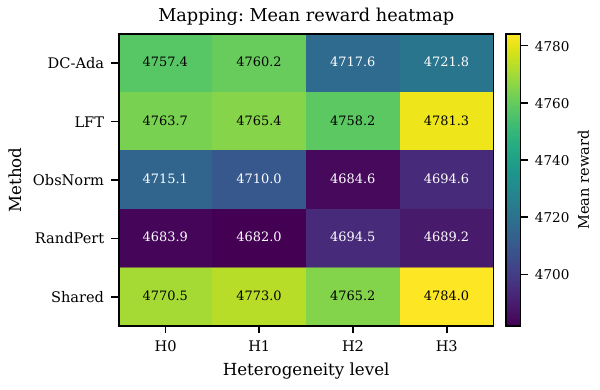}
\caption{Reward heatmaps across methods and heterogeneity levels. Uniform color across a row indicates robustness across H0--H3 for that environment.}
\label{fig:heatmap_reward}
\end{figure*}

\begin{figure*}[ht]
\centering
\includegraphics[width=0.30\textwidth]{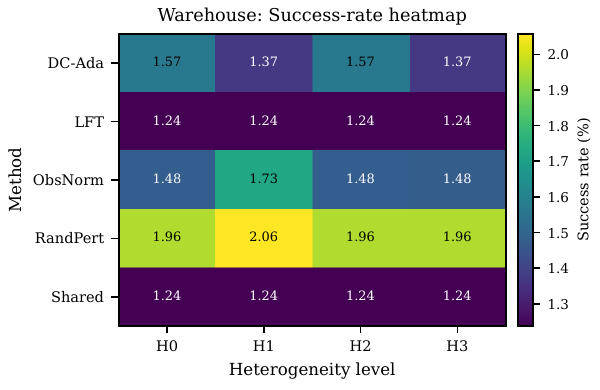}
\includegraphics[width=0.30\textwidth]{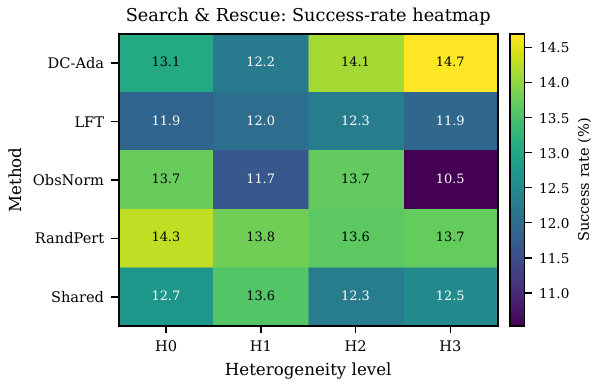}
\includegraphics[width=0.30\textwidth]{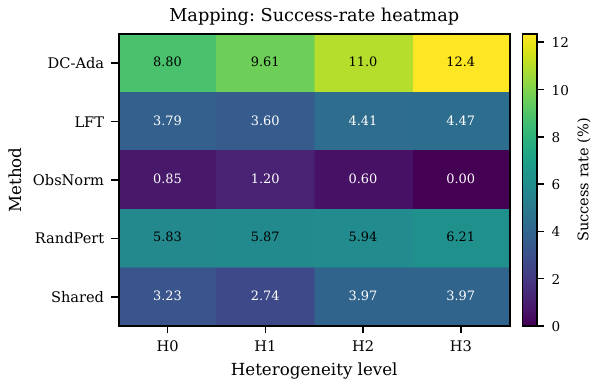}
\caption{Success-rate heatmaps across methods and heterogeneity levels using the thresholds in Table~\ref{tab:setup}. Sec.~\ref{sec:results_h3} contextualizes these values using continuous progress and threshold sensitivity under H3.}
\label{fig:heatmap_success}
\end{figure*}

\subsection{Severe heterogeneity (H3): progress and threshold sensitivity}
\label{sec:results_h3}
Under severe heterogeneity, thresholded success can be difficult to interpret in isolation. Table~\ref{tab:progress_h3} therefore reports continuous progress metrics together with success for H3, including delivery ratios (warehouse), rescue ratios (search-and-rescue), and coverage (mapping). These metrics clarify how far episodes typically are from the completion threshold. For example, under H3 the mean warehouse delivery ratio remains below $0.09$ even for the strongest methods, corresponding to fewer than $0.18$ deliveries per episode against a target of two; under H3 search-and-rescue the mean rescue ratio is approximately $0.29$--$0.31$, corresponding to approximately $0.58$--$0.62$ rescues per episode against a target of two; and under H3 mapping the mean coverage is approximately $0.62$--$0.68$ against a threshold of $0.75$.

Figure~\ref{fig:threshold_sensitivity} complements these summaries by sweeping the success threshold. In warehouse and search-and-rescue, the sensitivity curves show that the dominant mass of the completion distribution lies below the ``two-event'' threshold (two deliveries or two rescues), and that thresholds beyond two are not attained in the sweep. In mapping, the sensitivity curves show a smoother decay in completion probability as the coverage threshold increases; DC-Ada consistently shifts this curve upward over the shared-policy baseline across a broad range of thresholds (e.g., $\tau \in [0.60,0.75]$ under H3), which supports the interpretation that DC-Ada improves completion probability in a threshold-robust manner rather than only at a single operating point.

\input{tables/progress_success_H3.tex}

\begin{figure*}[ht]
\centering
\includegraphics[width=0.30\textwidth]{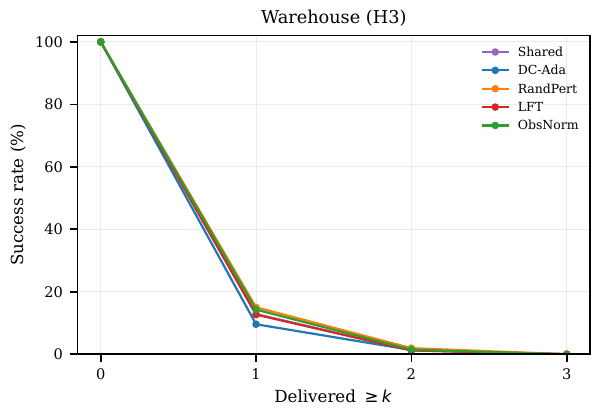}
\includegraphics[width=0.30\textwidth]{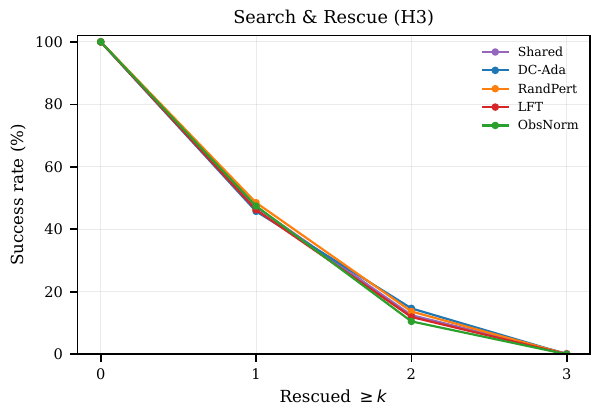}
\includegraphics[width=0.30\textwidth]{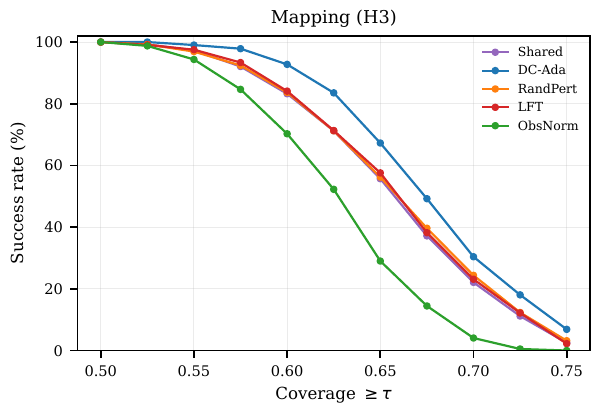}
\caption{Threshold sensitivity under severe heterogeneity (H3). Warehouse: delivered $\ge k$; Search-and-Rescue: rescued $\ge k$; Mapping: coverage $\ge \tau$. Curves indicate how completion probability varies with task strictness and contextualize rare-event success rates.}
\label{fig:threshold_sensitivity}
\end{figure*}

\subsection{Targeted ablations, observation-stress tests, and extended-seed confirmation}
\label{sec:results_targeted}
The main sweep provides a uniform five-seed comparison across all environments and heterogeneity levels. To strengthen the interpretation of the severe-heterogeneity results without changing that balanced main design, we add three targeted H3 checks using the same frozen policies and matched $200{,}000$-step budget. The first suite ablates common-random-number (CRN) evaluation, the number of candidate perturbations, the conservative accept/reject rule, and the truncated candidate-rollout horizon. The second suite evaluates clean H3 alongside additional observation artifacts---Gaussian observation noise, random feature dropout, a two-step sensing delay, and a mild combined perturbation. These two suites focus on mapping and search-and-rescue because the main sweep shows the clearest method- and metric-dependent behavior in those domains. The third check adds five new clean H3 seeds for each environment, yielding a ten-seed confirmation across warehouse logistics, search-and-rescue, and collaborative mapping. These follow-up studies are not a replacement for the balanced main sweep; they clarify design choices, robustness, and severe-heterogeneity reliability around the main findings while preserving the all-environment main comparison as the primary benchmark.

Table~\ref{tab:ablation_h3} reports the DC-Ada design ablations. In mapping, full DC-Ada achieves the highest completion rate among the tested variants (12.4\%) and the highest mean coverage, whereas removing CRN or using a single candidate reduces success to approximately 5.9\% and 5.8\%, respectively. The $M{=}4$ and always-best variants recover part of the completion improvement, but remain below full DC-Ada; shorter and longer candidate rollouts increase shaped reward but reduce completion probability. This confirms that the main mapping gain is not simply due to arbitrary perturbation: CRN-based comparison, multi-candidate selection, and conservative acceptance jointly support completion-oriented adaptation. In search-and-rescue, the picture is more mixed. Full DC-Ada improves success over the shared-policy baseline, but $M{=}4$ and the short-rollout variant achieve slightly higher success, while the long-rollout and no-CRN variants degrade. This reinforces the broader conclusion that DC-Ada's best settings are task-dependent rather than universally optimal.

\input{tables/ablation_h3_results.tex}

Table~\ref{tab:robustness_h3} reports the observation-stress suite. In mapping, DC-Ada retains the highest success rate under every tested artifact, with success remaining around 10.6\%--12.4\% compared with approximately 2.5\%--4.0\% for the shared policy and near-zero completion for observation normalization. Coverage shows the same qualitative pattern. In search-and-rescue, the stress results are more scenario-dependent: DC-Ada is strongest under the clean and delay conditions, but observation normalization is strongest under Gaussian noise, dropout, and the mild combined perturbation. These results strengthen the mapping-specific robustness claim while preserving the central paper-level interpretation: observation-interface adaptation is useful under severe sensing mismatch, but the best mitigation strategy still depends on the task, metric, and perturbation type.

\input{tables/robustness_h3_results.tex}

Table~\ref{tab:extra_seed_h3} reports the extended-seed confirmation. The all-environment tables above retain the original five-seed H0--H3 sweep because that sweep is balanced across every task and heterogeneity level; the extended table is a clean-H3 reliability check for the five main methods across all three environments. In warehouse, the additional seeds preserve the main-sweep trend: observation normalization remains strongest in shaped reward, while all methods continue to show low two-delivery completion rates. Random perturbation attains the highest rare completion/progress in this check, but with a large reward penalty, so it should not be interpreted as the strongest warehouse method overall. In search-and-rescue, the extended set reinforces the non-universal pattern: the shared policy remains strongest in reward, while random perturbation and DC-Ada are the strongest success-oriented methods within a narrow range. In mapping, the additional seeds preserve the central conclusion: DC-Ada remains the strongest method for completion and coverage over ten seeds, while the shared policy remains strongest in shaped reward. Thus, the extra seeds improve reporting symmetry and strengthen the paper's broader conclusion that the best adaptation strategy depends on task structure and metric rather than confirming a uniformly dominant method.

\input{tables/extra_seed_h3_results_all_envs.tex}

\subsection{Runtime, episode counts, and scalar-feedback proxy}
\label{sec:results_overhead}
Table~\ref{tab:overhead} reports measured runtime and scalar-feedback accounting under the fixed $B{=}200{,}000$-step budget. Because DC-Ada allocates interaction to short candidate evaluations (Alg.~\ref{alg:dcada}), it executes substantially fewer full-length episodes per run than baselines that use only nominal rollouts. Averaged over environments and heterogeneity levels, DC-Ada executes approximately $104$ full episodes per run, whereas non-adaptive baselines execute approximately $416$--$425$ episodes per run; local fine-tuning executes approximately $331$ episodes per run. This gap is an expected consequence of charging candidate rollouts and gradient rollouts against the same environment-step budget, and should be interpreted as a trade-off between using interaction for interface optimization versus using interaction purely for nominal episode execution.

Measured wall-clock time remains comparable across methods because all methods are budget-matched in environment steps: the median per-run time is approximately $11$--$12$ minutes on the evaluation hardware, with throughput on the order of $275$--$295$ steps/s (Table~\ref{tab:overhead}). The final column is a simple experiment-runner proxy: each rollout contributes one 64-bit scalar team-return value used for evaluation or logging. It should therefore be read as a low-bandwidth accounting convention rather than as literal on-wire traffic required by every baseline during deployment. Even under this conservative convention, the total volume remains only a few kilobytes per run, and DC-Ada avoids exchanging observations, maps, or gradients.

\section{Discussion, Implications, and Limitations}
\label{sec:discussion}

This section consolidates the empirical findings and clarifies the operational regime in which DC-Ada is expected to be effective. The overarching interpretation is that DC-Ada functions as a \emph{deploy-time compatibility layer}: it aims to preserve the coordination behavior encoded by a pretrained shared policy while compensating for heterogeneity-induced shifts in the observation distribution through robot-specific interface adaptation. This framing is intentionally narrower than a claim of universal superiority: the present evidence supports a task-dependent operating point under controlled simulation rather than a blanket replacement for simpler or gradient-based baselines.

\subsection{Key empirical findings and implications}
\label{subsec:disc_takeaways}
Across the three domains, the results support four main takeaways. First, heterogeneity can substantially reduce the performance of a fixed shared policy even when that policy is competent under homogeneous sensing. This indicates that a single shared controller does not automatically confer robustness to sensor mismatch, modality loss, or quality degradation, and motivates deploy-time mechanisms that control the policy's input distribution. Second, no single mitigation strategy dominates across all tasks and metrics. Observation normalization is strongest for reward robustness in warehouse logistics and remains competitive in search-and-rescue; the frozen shared policy is strongest for reward in collaborative mapping, with local fine-tuning close behind; and DC-Ada provides its clearest empirical advantage in mapping success and in some severe-heterogeneity completion metrics. Third, reward and completion need not rank methods in the same order. In mapping, for example, the strongest reward method is not the strongest success method, which shows that shaped return and thresholded completion capture different aspects of behavior. Fourth, the value of interface adaptation is task-structured. When the objective is closely tied to smooth progress, gains can appear in both reward and continuous progress variables; when completion depends on multiple discrete events, improvements may appear first as increased \emph{progress} and only later as increased \emph{success} for a fixed threshold. The targeted H3 suites further show that the design ingredients of DC-Ada matter most clearly in mapping, while search-and-rescue remains more sensitive to the specific perturbation, candidate-evaluation schedule, and seed set. The ten-seed H3 confirmation preserves the mapping completion advantage, confirms warehouse as a reward-normalization-dominated case, and does not convert search-and-rescue into a uniform DC-Ada win. This is consistent with the task-dependent interpretation rather than an algorithm-wide dominance claim.

These observations also carry methodological implications for evaluation. Reporting mean return alone can conceal whether a method shifts the distribution of completion outcomes, while reporting success rate alone can understate meaningful improvements when success is a tail event. Consequently, a comprehensive evaluation should include both thresholded success and continuous progress variables (e.g., delivered/rescued counts and coverage), together with clear documentation of the interaction budget and any additional rollouts used for adaptation.

\subsection{When does DC-Ada help?}
\label{subsec:disc_when_help}
DC-Ada is most appropriate in settings where (i) the learning signal available at runtime is limited to scalar team reward, (ii) gradient information is unavailable or unreliable (for example, due to black-box controllers, proprietary policies, or deployment constraints), and (iii) adaptation must remain lightweight and stable. Under these conditions, adapting only a small observation transform can correct systematic input mismatches while reducing the risk of destabilizing a validated coordination policy. The reported results suggest that this operating regime is particularly relevant when the practical objective is to preserve completion probability or coverage under sensing mismatch, rather than to maximize shaped reward in every domain. This operating regime is complementary to learned communication approaches and centralized coordination methods that rely on richer message passing or centralized training infrastructure \cite{Yang2024TeamComm,zhang2023commnetx}. In particular, when communication bandwidth is constrained, reward-only interface adaptation offers a pragmatic mechanism to improve robustness without introducing a learned inter-robot protocol.

\subsection{Dependence on pretrained policy design}
\label{subsec:policy_design_dependence}
A central caveat is that DC-Ada adapts the observation stream presented to a policy; it does not create new coordination skills absent from that policy. Consequently, the observed behavior depends on the inductive biases, reward shaping, exploration coverage, and architecture of the pretrained shared controller. If the frozen policy never learned a behavior needed for a particular task, or if the shaped reward favors behaviors that are only weakly aligned with the completion predicate, an observation-interface transform can have limited effect. This caveat helps explain the scenario-specific outcomes reported above: DC-Ada is most useful when performance loss is caused by systematic observation mismatch, whereas simpler statistics-based corrections or gradient-based transform updates can be preferable when the dominant issue is feature scaling, reward shaping, or policy-capability limitations.

\subsection{Failure modes and trade-offs}
\label{subsec:disc_failure_modes}
DC-Ada is conservative by design, and its limitations follow directly from that conservatism. The main observed failure cases in the present evaluation are informative. In warehouse logistics, observation normalization is clearly stronger in reward, indicating that the dominant heterogeneity effect is largely input-statistical rather than a case where reward-only search over transforms is sample-efficient. In collaborative mapping, DC-Ada improves completion-sensitive metrics but does not maximize shaped reward, showing that interface adaptation can shift completion probability without necessarily optimizing every component of the reward. Finally, in discrete-event domains such as warehouse and search-and-rescue, thresholded success can remain low even when progress metrics improve, because the success thresholds lie in the tail of the completion distribution.

\paragraph{Settings where simpler or gradient-based baselines may be preferable.}
When heterogeneity primarily manifests as a stable change in feature scale or normalization, simple observation normalization can be sufficient and can outperform learned adaptation in reward-centric regimes, as seen in warehouse logistics and parts of search-and-rescue. Conversely, when gradients through the frozen policy remain stable and compute is available, local fine-tuning can optimize transform parameters more directly and remain highly competitive in reward, which is reflected by its near-top mapping results. These cases do not weaken the motivation for DC-Ada; they clarify that algorithm choice should follow the available feedback, compute, and deployment constraints.

\paragraph{Sparse, delayed, or weakly informative reward.}
When the reward provides limited information about incremental progress, zeroth-order accept/reject updates may require many evaluations to identify beneficial directions, and improvements under fixed interaction budgets may be modest. In such regimes, even minimal auxiliary signals that are still communication-light (e.g., a small set of progress counters) can be substantially more sample-efficient than reward alone.

\paragraph{Strong coupling among robots.}
DC-Ada adapts transforms per robot, while the objective is the team return. If the task requires tightly coupled, coordinated changes across multiple robots, coordinate-like updates can stall or exhibit oscillatory behavior. Common-random-number (CRN) evaluation and conservative acceptance margins reduce spurious updates, but coupling remains a structural challenge for decentralized adaptation under scalar feedback.

\paragraph{Interaction and compute overhead.}
DC-Ada expends interaction to evaluate candidate transforms. This is a principled robustness--efficiency trade-off, but it implies that under a fixed interaction budget DC-Ada may execute fewer full-length episodes than a non-adaptive shared-policy baseline. For this reason, the paper reports interaction budgeting and runtime/throughput explicitly, rather than assuming identical episode counts across methods.

\paragraph{Transform expressiveness versus stability.}
More expressive transforms can better compensate for heterogeneity but also risk shifting inputs outside the distribution on which the shared policy was trained. This motivates identity-preserving initialization and bounded, conservative updates. Conceptually, DC-Ada is not intended to synthesize new behaviors; it is intended to maintain the pretrained policy in a regime where its behaviors remain valid under heterogeneous sensing.

\subsection{Success metrics, rare events, and task structure}
\label{subsec:disc_success_metrics}
A recurring phenomenon in heterogeneous multi-robot evaluation is that thresholded success can remain rare even when progress improves materially. This occurs because success is typically defined as a thresholded indicator of an underlying progress variable at episode end. Let $z_T$ denote an episode-level progress variable such as deliveries completed, rescues completed, or coverage achieved at the terminal time $T$, and let $\tau$ denote the corresponding success threshold. Then success can be expressed as:
\begin{equation*}
S_{\tau} \;=\; \mathbb{I}\{z_T \ge \tau\}.
\end{equation*}
Under severe heterogeneity, the distribution of $z_T$ may shift to lower values, and an improvement in $\mathbb{E}[z_T]$ does not necessarily yield a commensurate increase in $\Pr[z_T \ge \tau]$ when $\tau$ lies in the tail. This motivates the reporting strategy adopted in this paper: success rate is reported alongside continuous progress metrics, and sensitivity analyses with respect to $\tau$ are used to contextualize success under different completion thresholds. Interpreting these signals jointly is particularly important for tasks that require multiple discrete subgoals before any episode counts as successful.

\subsection{Simulator fidelity and external validity}
\label{subsec:disc_fidelity}
The simulator used in this work is intentionally lightweight (pure NumPy) to promote portability, controlled ablations, and broad reproducibility. The observation-stress suite adds Gaussian noise, dropout, sensing delay, and a mild combined perturbation, which provides an additional check that the H3 mapping findings are not limited to a perfectly clean observation stream. Nevertheless, the simulator does not capture the full set of effects present in high-fidelity robotics simulation or physical deployment. In particular, richer physics and sensing can introduce contact dynamics, actuation limits and latency, wheel slip, asynchronous sensor update rates, calibration drift, and perception artifacts beyond those tested here. These factors can influence both the severity of heterogeneity-induced degradation and the stability of online adaptation.

A natural extension is to evaluate DC-Ada in higher-fidelity simulators and on physical robot platforms while preserving the fixed-interface adaptation principle. In addition, integrating interface adaptation with navigation stacks that incorporate planning, SLAM, and mapping back-ends raises substantive system questions about where adaptation is best applied and which intermediate signals should be exposed. DC-Ada is compatible with these stacks in principle, either as an adapter on raw/encoded sensor features or as an adapter on intermediate representations produced by perception or mapping modules \cite{carreno2022planning,chang2021kimera,chen2022slam,goarin2024graph}. Establishing these integrations rigorously is an important step toward deployment-oriented validation.

\section{Conclusion}
\label{sec:conclusion}

This work presented \textbf{DC-Ada}, a reward-only decentralized observation-interface adaptation method for heterogeneous multi-robot teams. The method freezes a pretrained shared policy and adapts only per-robot observation transforms using a conservative accept/reject zeroth-order procedure. This design targets a practically relevant regime in which gradient information is unavailable, communication is constrained, and stability is prioritized: rather than retraining coordination policies or learning message-passing protocols, DC-Ada mitigates heterogeneity by correcting the policy input distribution at the interface.

Empirically, the results show that heterogeneous sensing can significantly degrade fixed shared policies, but they do not indicate a universal winner across all domains and metrics. Instead, they show that the most effective mitigation strategy is task-dependent. Observation normalization can be strongest in some reward-centric regimes, gradient-based baselines can remain highly competitive when gradients are stable, and DC-Ada offers a complementary operating point with its clearest gains in mapping completion under severe heterogeneity while maintaining reward-only, communication-light adaptation. The targeted ablations show that CRN evaluation, multi-candidate search, and conservative acceptance are important for DC-Ada's mapping-completion gains; the observation-stress tests show that those gains persist under several additional observation artifacts; and the extended ten-seed H3 confirmation preserves the mapping success/coverage trend while showing that warehouse remains dominated by observation normalization in reward and that search-and-rescue remains sensitive to seed and baseline choice. The analysis also clarifies that success rates may remain low under strict completion thresholds even when progress improves, particularly in tasks that require multiple discrete events to declare success. For this reason, the paper reports thresholded success alongside continuous progress variables and provides sensitivity analyses that characterize how completion probability varies with the success threshold.

Beyond the algorithmic contribution, this paper provides a reproducible evaluation pipeline that makes heterogeneity measurable, including a fixed-dimensional observation interface, transparent success definitions, and runtime/interaction accounting. At the same time, because the present study uses a lightweight two-dimensional simulator, the conclusions should be read as controlled evidence about the deploy-time interface-adaptation principle rather than as final deployment validation.

Several research directions follow naturally from this study. First, evaluating DC-Ada under higher-fidelity physics and sensor models, and ultimately on physical robot platforms, is essential to quantify robustness under latency, asynchrony, and real perception artifacts. Second, the transform family can be broadened to structured adapters (e.g., modality gating or low-rank parameterizations) while maintaining stability constraints that preserve the validity of the pretrained policy. Third, adaptation schedules can be made state- or event-dependent to reduce unnecessary evaluations and improve responsiveness to abrupt failures. Finally, it is of interest to develop tighter theoretical characterizations for conservative accept/reject adaptation under truncated rollouts and coupled multi-agent objectives, including explicit bounds on false accept probability and the role of common-random-number evaluations in reducing selection variance. These directions preserve the central principle of DC-Ada---communication-light, reward-only adaptation at the observation interface---while strengthening its applicability to realistic multi-robot deployments.

%\section*{Acknowledgment}
%To be updated.

%\bibliographystyle{IEEEtran}%unsrt} 
%\bibliography{references}
% Generated by IEEEtran.bst, version: 1.14 (2015/08/26)

\begin{IEEEbiographynophoto}{Saad Alqithami} received the M.Sc. (2012) and Ph.D. (2016) degrees in Computer Science from Southern Illinois University Carbondale, USA. He is currently an Associate Professor in the Department of Computer Science at Al-Baha University, Saudi Arabia. His research interests include artificial intelligence, agents \& multi-agent systems, network science, social network analysis, and applied machine learning for resilient systems, with applications spanning healthcare, sustainability, and tourism. He has led funded research and innovation projects, and he actively serves the research community through editorial and peer-review activities.
%is an Associate Professor of Computer Science at Al-Baha University, Saudi Arabia, a position he has held since 2022. He earned his PhD and MSc degrees in Computer Science from Southern Illinois University, Carbondale, United States. Dr. Alqithami's research spans the fields of artificial intelligence, network science, multiagent systems, and social networks, with a strong focus on their applications in health informatics, innovative computational models, and intelligent systems. Dr. Alqithami has an extensive publication record in high-impact, peer-reviewed journals and has presented his work at leading international conferences. His research contributions include advancements in AI-driven predictive models, network analysis techniques, and the development of multiagent frameworks to address complex, real-world problems. In addition to his academic research, Dr. Alqithami is deeply involved in promoting interdisciplinary collaboration. He actively reviews grant proposals, serves on editorial boards of esteemed journals in artificial intelligence and computer science, and mentors students and early-career researchers. His leadership roles in academic and professional settings reflect his commitment to fostering innovation and driving the evolution of AI to address societal challenges.
\end{IEEEbiographynophoto}

\EOD
\end{document}

%% file: tables/obs_layout_table.tex
\begin{table}[t]
\centering \scriptsize
\caption{Fixed-layout observation interface. All robots receive the same observation dimension within a task, regardless of sensor suite. Missing modalities are filled with neutral defaults. LiDAR rays are resampled to a fixed segment length of $16$.}
\label{tab:obs_layout}
\resizebox{\linewidth}{!}{
\setlength{\tabcolsep}{4pt}
\begin{tabular}{@{}lccl@{}}
\toprule
Task & \shortstack{Extra\\ features $k$} & \shortstack{Total\\ dim $d$} & Layout \\
\midrule
Warehouse & 5 & 73 & $[p(2),\, v(2),\, k,\, \mathrm{lidar}(16),\, \mathrm{rgb}(32),\, \mathrm{depth}(16)]$ \\
Search \& Rescue & 3 & 71 & $[p(2),\, v(2),\, k,\, \mathrm{lidar}(16),\, \mathrm{rgb}(32),\, \mathrm{depth}(16)]$ \\
Mapping & 2 & 70 & $[p(2),\, v(2),\, k,\, \mathrm{lidar}(16),\, \mathrm{rgb}(32),\, \mathrm{depth}(16)]$ \\
\bottomrule
\end{tabular}}
\end{table}

%% file: tables/heterogeneity_table.tex
\begin{table}[t]
\centering \scriptsize
\caption{Sensor heterogeneity configurations (H0--H3) used to instantiate the four robots. When a modality is absent, the corresponding observation segment is filled with neutral defaults (zeros for RGB; ones for LiDAR/depth distance-like segments) to preserve a fixed observation layout across heterogeneity levels. Camera field-of-view (FOV) is $60^\circ$ unless specified (H1).}
\label{tab:heterogeneity}
\begin{tabular*}{\linewidth}{@{\extracolsep{\fill}}llcccc@{}}
\toprule
\textbf{Level} & \textbf{Robot} & \textbf{LiDAR} & \textbf{Range} & \textbf{Rays} & \textbf{RGB / Depth (FOV)} \\
\midrule
H0 & R1--R4 & yes & 5.0 & 16 & RGB+Depth ($60^\circ$) \\
\midrule
\multirow{4}{*}{H1}   & R1 & yes & 4.0 & 16 & RGB+Depth ($50^\circ$) \\
   & R2 & yes & 5.0 & 16 & Depth only ($60^\circ$) \\
   & R3 & yes & 6.0 & 16 & RGB+Depth ($70^\circ$) \\
   & R4 & yes & 4.0 & 16 & Depth only ($80^\circ$) \\
\midrule
\multirow{4}{*}{H2}   & R1 & yes & 6.0 & 16 & RGB only ($60^\circ$) \\
   & R2 & yes & 4.0 & 24 & Depth only ($60^\circ$) \\
   & R3 & no & -- & -- & RGB+Depth ($60^\circ$) \\
   & R4 & yes & 5.0 & 16 & RGB+Depth ($60^\circ$) \\
\midrule
\multirow{4}{*}{H3}   & R1 & yes & 8.0 & 32 & none \\
   & R2 & no & -- & -- & RGB only ($60^\circ$) \\
   & R3 & no & -- & -- & Depth only ($60^\circ$) \\
   & R4 & yes & 3.0 & 8 & RGB+Depth ($60^\circ$) \\
\bottomrule
\end{tabular*}
\end{table}

%% file: tables/hyperparams_table_access.tex
\begin{table}[t]
\centering \scriptsize
\caption{Key hyperparameters used in the final reported sweep. ``--'' indicates the parameter is not applicable. The shared policy is frozen during adaptation, and all methods respect the same matched interaction budget of $B{=}200{,}000$ joint environment steps per run.}
\label{tab:hyperparams}
\begin{tabular*}{\linewidth}{@{\extracolsep{\fill}}p{.25\linewidth}ccccc@{}}
\toprule
Hyperparameter & DC-Ada & \shortstack{Random\\ Perturb.} & \shortstack{Obs.\\ Norm.} & \shortstack{Local \\Fine-Tune} & \shortstack{Shared\\ Policy} \\
\midrule
Update interval $K$ (episodes) & 3 & 3 & -- & 1 & -- \\
Candidates $M$ (per update) & 8 & -- & -- & -- & -- \\
Noise scale $\sigma$ & 0.05 & 0.05 & -- & -- & -- \\
Step size $\alpha$ & 1.0 & -- & -- & -- & -- \\
Short rollout fraction & 0.25 & -- & -- & 0.25 & -- \\
Accept margin $\tau$ & 0.0 & -- & -- & -- & -- \\
CRN-style repeated seed evaluation & enabled & disabled & -- & enabled & -- \\
Running mean/variance normalization & -- & -- & enabled & optional & -- \\
Learning rate (transform updates) & -- & -- & -- & $3\times 10^{-4}$ & -- \\
Gradient steps per episode & -- & -- & -- & 1 & -- \\
Discount $\gamma$ & -- & -- & -- & 0.99 & -- \\
Entropy coefficient & -- & -- & -- & 0.01 & -- \\
Gradient clipping & -- & -- & -- & 1.0 & -- \\
\bottomrule
\end{tabular*}
\end{table}

%% file: tables/setup_table_revised.tex
\begin{table}[t]
\centering \scriptsize
\caption{Task specifications, termination, and success criteria used in the final reported sweep. Each evaluation run uses $N{=}4$ robots, a horizon of $T{=}500$ steps per episode, and a fixed interaction budget of $B{=}200{,}000$ joint environment steps. The success criteria are deliberately non-trivial operational targets rather than one-event minima.}
\label{tab:setup}
\begin{tabular*}{\linewidth}{@{\extracolsep{\fill}}p{.2\linewidth}p{.2\linewidth}p{.2\linewidth}p{.2\linewidth}@{}}
\toprule
Task & Progress metric & Success criterion & Episode termination \\
\midrule
Warehouse logistics & delivery ratio $=\text{delivered}/2$ & delivered $\ge 2$ & success or $T{=}500$ steps \\
Search and rescue & rescue ratio $=\text{rescued}/2$ & rescued $\ge 2$ & success or $T{=}500$ steps \\
Collaborative mapping & coverage & coverage $\ge 0.75$ & success or $T{=}500$ steps \\
\bottomrule
\end{tabular*}
\end{table}

%% file: tables/reward_table.tex
\begin{table}[t]
\centering \scriptsize
\caption{Reward functions used in the three domains (team-level scalar reward). All event terms use counter deltas to avoid reward explosion. Symbols follow the environment implementations.}
\label{tab:reward}
\begin{tabular*}{\linewidth}{@{\extracolsep{\fill}}p{.1\linewidth}p{.8\linewidth}@{}}
\toprule
Task & Reward per step $r_t$ \\
\midrule
Warehouse & $r_t = -0.1 + 100\,\Delta\mathrm{delivered} + 1\,\Delta\mathrm{picked\_up} - 10\,\Delta\mathrm{collisions} + 0.01\sum_{i=1}^{N}(20-d_i)$, where $d_i$ is the distance from robot $i$ to the nearest relevant target (package if not carrying; drop-off otherwise). \\
Search \& Rescue & $r_t = -0.1 + 5\,\Delta\mathrm{found} + 20\,\Delta\mathrm{rescued} + 0.1\,\mathrm{health\_bonus} + 0.5\,\mathrm{coverage}$, where $\mathrm{coverage}$ is computed on a coarse grid and $\mathrm{health\_bonus}$ credits remaining victim health at rescue time. \\
Mapping & $r_t = -0.1 + 10\,\mathrm{coverage} + 50\,\Delta\mathrm{coverage} + 0.5\,\mathrm{spread}$, where $\mathrm{coverage}$ is the explored-cell ratio and $\mathrm{spread}$ is the mean distance to the team centroid. \\
\bottomrule
\end{tabular*}
\end{table}

%% file: tables/overhead_table_access_revised.tex
\begin{table}[t]
\centering \scriptsize
\caption{Runtime and scalar-feedback accounting for the final reported $B{=}200{,}000$-step sweep, aggregated over all environments, heterogeneity levels, and seeds. Runtime is measured wall-clock time per run on the evaluation hardware; throughput is median environment steps per second; the last column is an experiment-runner proxy that counts one 64-bit scalar team-return value per rollout. It should be interpreted as a low-bandwidth accounting convention rather than literal required deployment traffic for every baseline.}
\label{tab:overhead}
\resizebox{\linewidth}{!}{
\begin{tabular}{@{}lcccc@{}}
\toprule
Method & \shortstack{Mean \\episodes/run} & \shortstack{Median \\runtime (s)} & \shortstack{Median \\throughput (steps/s)} & \shortstack{Scalar-feedback\\proxy (bytes/run)} \\
\midrule
Shared Policy & 417.2 & 698.3 & 286.7 & 3338 \\
DC-Ada & 104.2 & 702.0 & 285.0 & 10693 \\
Random Perturb. & 425.0 & 727.7 & 275.1 & 3400 \\
Local Fine-Tune & 331.4 & 712.7 & 281.1 & 5302 \\
Obs. Norm. & 415.6 & 678.5 & 295.2 & 3325 \\
\bottomrule
\end{tabular}}
\end{table}

%% file: tables/results_table_access.tex
\begin{table*}[t]
\centering \scriptsize
\caption{Performance comparison across heterogeneity levels under a fixed budget of $B{=}200{,}000$ environment steps. Values are mean episode reward $\pm$ std over 5 seeds.}
\label{tab:results}
\begin{tabular*}{\linewidth}{@{\extracolsep{\fill}}llcccc@{}}
\toprule
\textbf{Environment} & \textbf{Method} & \textbf{H0} & \textbf{H1} & \textbf{H2} & \textbf{H3} \\
\midrule
\multirow{5}{*}{Mapping} & DC-Ada (Ours) & 4757.4 $\pm$ 86.8 & 4760.2 $\pm$ 96.2 & 4717.6 $\pm$ 62.2 & 4721.8 $\pm$ 55.3 \\
 & Local Fine-Tuning & 4763.7 $\pm$ 0.4 & 4765.4 $\pm$ 0.6 & 4758.2 $\pm$ 0.4 & 4781.3 $\pm$ 0.5 \\
 & Obs. Normalization & 4715.1 $\pm$ 0.9 & 4710.0 $\pm$ 1.5 & 4684.6 $\pm$ 2.0 & 4694.6 $\pm$ 0.8 \\
 & Random Perturbation & 4683.9 $\pm$ 51.3 & 4682.0 $\pm$ 53.0 & 4694.5 $\pm$ 49.0 & 4689.2 $\pm$ 58.1 \\
 & Shared Policy & \textbf{4770.5 $\pm$ 0.3} & \textbf{4773.0 $\pm$ 0.3} & \textbf{4765.2 $\pm$ 0.8} & \textbf{4784.0 $\pm$ 0.8} \\
\midrule
\multirow{5}{*}{Search \& Rescue} & DC-Ada (Ours) & 14.1 $\pm$ 2.0 & 13.5 $\pm$ 1.1 & 15.2 $\pm$ 2.2 & 15.2 $\pm$ 1.0 \\
 & Local Fine-Tuning & 14.4 $\pm$ 0.3 & 14.5 $\pm$ 0.2 & 14.6 $\pm$ 0.1 & 14.8 $\pm$ 0.1 \\
 & Obs. Normalization & \textbf{17.7 $\pm$ 0.2} & 16.2 $\pm$ 0.1 & \textbf{17.2 $\pm$ 0.1} & 15.9 $\pm$ 0.2 \\
 & Random Perturbation & 13.8 $\pm$ 1.9 & 13.7 $\pm$ 2.0 & 13.5 $\pm$ 2.2 & 13.6 $\pm$ 1.7 \\
 & Shared Policy & 16.0 $\pm$ 0.1 & \textbf{16.5 $\pm$ 0.0} & 15.7 $\pm$ 0.1 & \textbf{16.1 $\pm$ 0.1} \\
\midrule
\multirow{5}{*}{Warehouse} & DC-Ada (Ours) & 254.9 $\pm$ 1.5 & 254.1 $\pm$ 1.5 & 253.9 $\pm$ 2.0 & 253.0 $\pm$ 2.2 \\
 & Local Fine-Tuning & 253.8 $\pm$ 0.2 & 253.7 $\pm$ 0.1 & 252.4 $\pm$ 0.2 & 250.9 $\pm$ 0.2 \\
 & Obs. Normalization & \textbf{269.5 $\pm$ 0.1} & \textbf{270.8 $\pm$ 0.1} & \textbf{269.7 $\pm$ 0.2} & \textbf{268.9 $\pm$ 0.1} \\
 & Random Perturbation & 234.4 $\pm$ 3.7 & 235.0 $\pm$ 3.6 & 234.3 $\pm$ 3.3 & 234.0 $\pm$ 3.3 \\
 & Shared Policy & 254.7 $\pm$ 0.0 & 254.5 $\pm$ 0.0 & 253.5 $\pm$ 0.0 & 251.6 $\pm$ 0.0 \\
\bottomrule
\end{tabular*}
\end{table*}

%% file: tables/mapping_contrasts_table.tex
\begin{table}[t]
\centering \scriptsize
\caption{Selected seed-paired contrasts for collaborative mapping under severe heterogeneity (H3; $n{=}5$ seeds). Positive values favor the first method named in the contrast. ``pp'' denotes percentage points.}
\label{tab:mapping_contrasts}
\begin{tabular*}{\linewidth}{@{\extracolsep{\fill}}lll@{}}
\toprule
\textbf{Metric} & \textbf{Contrast} & \textbf{Mean $\Delta$ [95\% CI]} \\
\midrule
Reward & Shared Policy $-$ DC-Ada & $+62.2$ $[-15.1,\ 139.5]$ \\
Success rate & DC-Ada $-$ Shared Policy & $+8.38$ pp $[6.75,\ 10.02]$ pp \\
Coverage & DC-Ada $-$ Shared Policy & $+0.0190$ $[0.0159,\ 0.0220]$ \\
\bottomrule
\end{tabular*}
\end{table}

%% file: tables/progress_success_H3.tex
\begin{table}[t]
\centering \scriptsize
\caption{H3 progress and success under severe heterogeneity. Mean $\pm$ std over 5 seeds. Progress is delivery ratio (Warehouse), rescue ratio (Search \& Rescue), and coverage (Mapping). Success uses the thresholds in Table~\ref{tab:setup}.}
\label{tab:progress_h3}
\begin{tabular*}{\linewidth}{@{\extracolsep{\fill}}llcc@{}}
\toprule
\textbf{Environment} & \textbf{Method} & \textbf{Progress} & \textbf{Success (\%)} \\
\midrule
\multirow{5}{*}{Mapping} & DC-Ada (Ours) & 0.675 $\pm$ 0.002 & 12.4 $\pm$ 1.2 \\
     & Local Fine-Tuning & 0.658 $\pm$ 0.000 & 4.5 $\pm$ 0.2 \\
     & Obs. Normalization & 0.625 $\pm$ 0.000 & 0.0 $\pm$ 0.0 \\
     & Random Perturbation & 0.659 $\pm$ 0.005 & 6.2 $\pm$ 1.8 \\
     & Shared Policy & 0.656 $\pm$ 0.000 & 4.0 $\pm$ 0.0 \\
\midrule
\multirow{5}{*}{Search \& Rescue} & DC-Ada (Ours) & 0.302 $\pm$ 0.013 & 14.7 $\pm$ 1.1 \\
     & Local Fine-Tuning & 0.291 $\pm$ 0.001 & 11.9 $\pm$ 0.2 \\
     & Obs. Normalization & 0.290 $\pm$ 0.003 & 10.5 $\pm$ 0.3 \\
     & Random Perturbation & 0.311 $\pm$ 0.020 & 13.7 $\pm$ 1.7 \\
     & Shared Policy & 0.300 $\pm$ 0.001 & 12.5 $\pm$ 0.1 \\
\midrule
\multirow{5}{*}{Warehouse} & DC-Ada (Ours) & 0.055 $\pm$ 0.012 & 1.4 $\pm$ 0.5 \\
     & Local Fine-Tuning & 0.070 $\pm$ 0.000 & 1.2 $\pm$ 0.0 \\
     & Obs. Normalization & 0.079 $\pm$ 0.000 & 1.5 $\pm$ 0.0 \\
     & Random Perturbation & 0.085 $\pm$ 0.006 & 2.0 $\pm$ 0.3 \\
     & Shared Policy & 0.071 $\pm$ 0.000 & 1.2 $\pm$ 0.0 \\
\bottomrule
\end{tabular*}
\end{table}

%% file: tables/ablation_h3_results.tex
\begin{table*}[t]
\centering
\scriptsize
\caption{Targeted DC-Ada design ablations under severe heterogeneity (H3). All runs use the same frozen shared policy, five seeds, and a matched $200{,}000$-step budget. Mapping progress is mean coverage; search-and-rescue progress is mean rescue ratio. Values are mean $\pm$ standard deviation across seeds.}
\label{tab:ablation_h3}
\begin{tabular*}{\linewidth}{@{\extracolsep{\fill}}llcccc@{}}
\toprule
\textbf{Environment} & \textbf{Variant} & \textbf{Reward} & \textbf{Success (\%)} & \textbf{Progress} & \textbf{Episodes} \\
\midrule
\multirow{8}{*}{Mapping} & Shared policy & 4784.0 $\pm$ 0.9 & 4.0 $\pm$ 0.0 & 0.656 $\pm$ 0.000 & 403.0 $\pm$ 0.0 \\
 & Full DC-Ada & 4721.8 $\pm$ 61.9 & 12.4 $\pm$ 1.3 & 0.675 $\pm$ 0.003 & 102.0 $\pm$ 0.0 \\
 & No CRN & 4775.9 $\pm$ 9.7 & 5.9 $\pm$ 1.2 & 0.661 $\pm$ 0.003 & 102.0 $\pm$ 0.0 \\
 & $M{=}1$ & 4804.6 $\pm$ 49.9 & 5.8 $\pm$ 1.2 & 0.659 $\pm$ 0.002 & 243.6 $\pm$ 0.9 \\
 & $M{=}4$ & 4771.8 $\pm$ 77.9 & 9.2 $\pm$ 1.9 & 0.672 $\pm$ 0.002 & 153.4 $\pm$ 0.9 \\
 & Always-best & 4762.9 $\pm$ 100.3 & 9.6 $\pm$ 2.5 & 0.673 $\pm$ 0.004 & 102.0 $\pm$ 0.0 \\
 & Short $T_c$ & 4875.5 $\pm$ 62.4 & 7.5 $\pm$ 1.5 & 0.668 $\pm$ 0.004 & 186.4 $\pm$ 0.5 \\
 & Long $T_c$ & 4858.2 $\pm$ 28.9 & 8.4 $\pm$ 2.0 & 0.672 $\pm$ 0.003 & 59.8 $\pm$ 0.4 \\
\midrule
\multirow{8}{*}{Search--Rescue} & Shared policy & 16.09 $\pm$ 0.10 & 12.5 $\pm$ 0.1 & 0.300 $\pm$ 0.001 & 444.0 $\pm$ 0.0 \\
 & Full DC-Ada & 15.21 $\pm$ 1.08 & 14.7 $\pm$ 1.2 & 0.302 $\pm$ 0.014 & 108.8 $\pm$ 2.2 \\
 & No CRN & 12.59 $\pm$ 1.59 & 11.6 $\pm$ 1.9 & 0.271 $\pm$ 0.025 & 108.2 $\pm$ 0.4 \\
 & $M{=}1$ & 15.62 $\pm$ 0.32 & 13.9 $\pm$ 0.5 & 0.302 $\pm$ 0.005 & 265.8 $\pm$ 2.7 \\
 & $M{=}4$ & 15.82 $\pm$ 0.50 & 14.9 $\pm$ 0.9 & 0.310 $\pm$ 0.006 & 166.2 $\pm$ 3.5 \\
 & Always-best & 14.23 $\pm$ 2.83 & 14.4 $\pm$ 3.0 & 0.310 $\pm$ 0.033 & 110.4 $\pm$ 3.3 \\
 & Short $T_c$ & 14.59 $\pm$ 2.55 & 15.1 $\pm$ 1.6 & 0.298 $\pm$ 0.032 & 198.0 $\pm$ 2.5 \\
 & Long $T_c$ & 13.55 $\pm$ 1.91 & 11.0 $\pm$ 2.3 & 0.269 $\pm$ 0.023 & 63.6 $\pm$ 1.3 \\
\bottomrule
\end{tabular*}
\end{table*}

%% file: tables/robustness_h3_results.tex
\begin{table*}[t]
\centering
\scriptsize
\caption{Observation-stress evaluation under severe heterogeneity (H3). Each cell reports success rate (\%) / progress, where progress is coverage for mapping and rescue ratio for search-and-rescue. Values are means over five seeds under the matched $200{,}000$-step budget.}
\label{tab:robustness_h3}
\begin{tabular*}{\linewidth}{@{\extracolsep{\fill}}llcccc@{}}
\toprule
\textbf{Environment} & \textbf{Stress condition} & \textbf{Shared} & \textbf{DC-Ada} & \textbf{Local FT} & \textbf{Obs. norm.} \\
\midrule
\multirow{5}{*}{Mapping} & Clean & 4.0 / 0.656 & \textbf{12.4 / 0.675} & 4.5 / 0.658 & 0.0 / 0.625 \\
 & Gaussian noise & 4.0 / 0.656 & \textbf{11.4 / 0.675} & 4.5 / 0.658 & 0.0 / 0.633 \\
 & 10\% dropout & 2.5 / 0.652 & \textbf{10.6 / 0.673} & 2.8 / 0.653 & 0.0 / 0.627 \\
 & 2-step delay & 4.0 / 0.656 & \textbf{10.8 / 0.674} & 4.5 / 0.658 & 0.0 / 0.625 \\
 & Combined mild & 3.7 / 0.655 & \textbf{12.0 / 0.676} & 4.3 / 0.657 & 0.2 / 0.632 \\
\midrule
\multirow{5}{*}{Search--Rescue} & Clean & 12.5 / 0.300 & \textbf{14.7 / 0.302} & 11.9 / 0.291 & 10.5 / 0.290 \\
 & Gaussian noise & 12.5 / 0.300 & 13.5 / 0.299 & 12.3 / 0.294 & \textbf{13.6 / 0.310} \\
 & 10\% dropout & 13.2 / 0.305 & 12.5 / 0.280 & 11.8 / 0.292 & \textbf{13.3 / 0.306} \\
 & 2-step delay & 12.5 / 0.298 & \textbf{15.2 / 0.306} & 12.7 / 0.295 & 9.7 / 0.286 \\
 & Combined mild & 12.5 / 0.298 & 12.8 / 0.286 & 12.8 / 0.296 & \textbf{14.0 / 0.307} \\
\bottomrule
\end{tabular*}
\end{table*}

%% file: tables/extra_seed_h3_results_all_envs.tex
\begin{table*}[t]
\centering
\scriptsize
\caption{Extended-seed confirmation for clean H3 across all three environments. This table combines the original five clean H3 seeds with five additional seeds (10 seeds total) for the five main methods. Progress is delivery ratio for warehouse, rescue ratio for search--rescue, and coverage for mapping. The main all-environment tables retain the uniform five-seed H0--H3 sweep; this targeted check evaluates whether the severe-heterogeneity trends persist under additional seeds. Values are mean $\pm$ standard deviation across seeds.}
\label{tab:extra_seed_h3}
\begin{tabular*}{\linewidth}{@{\extracolsep{\fill}}llcccc@{}}
\toprule
\textbf{Environment} & \textbf{Method} & \textbf{n} & \textbf{Reward} & \textbf{Success (\%)} & \textbf{Progress} \\
\midrule
\multirow{5}{*}{Warehouse} & Shared policy & 10 & 251.5 $\pm$ 0.1 & 1.2 $\pm$ 0.0 & 0.071 $\pm$ 0.000 \\
 & DC-Ada (Ours) & 10 & 253.2 $\pm$ 1.7 & 1.4 $\pm$ 0.5 & 0.056 $\pm$ 0.011 \\
 & Random perturbation & 10 & 234.0 $\pm$ 3.0 & \textbf{2.0 $\pm$ 0.3} & \textbf{0.084 $\pm$ 0.005} \\
 & Local fine-tuning & 10 & 251.0 $\pm$ 0.3 & 1.3 $\pm$ 0.1 & 0.071 $\pm$ 0.002 \\
 & Obs. normalization & 10 & \textbf{268.9 $\pm$ 0.1} & 1.5 $\pm$ 0.0 & 0.079 $\pm$ 0.001 \\
\midrule
\multirow{5}{*}{Search--Rescue} & Shared policy & 10 & \textbf{16.05 $\pm$ 0.08} & 12.5 $\pm$ 0.1 & 0.299 $\pm$ 0.001 \\
 & DC-Ada (Ours) & 10 & 14.42 $\pm$ 1.36 & 13.4 $\pm$ 2.0 & 0.290 $\pm$ 0.021 \\
 & Random perturbation & 10 & 13.45 $\pm$ 1.51 & \textbf{13.8 $\pm$ 1.4} & \textbf{0.307 $\pm$ 0.016} \\
 & Local fine-tuning & 10 & 14.67 $\pm$ 0.11 & 11.8 $\pm$ 0.2 & 0.290 $\pm$ 0.002 \\
 & Obs. normalization & 10 & 15.94 $\pm$ 0.14 & 10.5 $\pm$ 0.3 & 0.291 $\pm$ 0.002 \\
\midrule
\multirow{5}{*}{Mapping} & Shared policy & 10 & \textbf{4783.0 $\pm$ 1.4} & 3.9 $\pm$ 0.1 & 0.656 $\pm$ 0.000 \\
 & DC-Ada (Ours) & 10 & 4736.6 $\pm$ 70.1 & \textbf{11.2 $\pm$ 2.4} & \textbf{0.675 $\pm$ 0.002} \\
 & Random perturbation & 10 & 4690.5 $\pm$ 51.6 & 6.0 $\pm$ 1.4 & 0.657 $\pm$ 0.005 \\
 & Local fine-tuning & 10 & 4779.7 $\pm$ 1.8 & 4.4 $\pm$ 0.2 & 0.658 $\pm$ 0.000 \\
 & Obs. normalization & 10 & 4694.6 $\pm$ 0.9 & 0.0 $\pm$ 0.0 & 0.624 $\pm$ 0.000 \\
\bottomrule
\end{tabular*}
\end{table*}